\documentclass[twoside]{article}
\usepackage[accepted]{aistats2027}
\usepackage{amsmath,amssymb,amsthm}
\usepackage{graphicx,booktabs,xcolor}
\usepackage{tikz,placeins}
\usepackage[round]{natbib}

\usepackage[hidelinks]{hyperref}
\newtheorem{theorem}{Theorem}
\newtheorem{proposition}{Proposition}
\newtheorem{corollary}{Corollary}
\newtheorem{lemma}{Lemma}

\makeatletter
\renewcommand{\@copyrightspace}{}
\makeatother

\renewcommand{\url}[1]{}
\newcommand{\E}{\mathbb E}

\newcommand{\N}{\mathcal N}
\newcommand{\Reg}{\operatorname{Reg}}
\newcommand{\tr}{\operatorname{tr}}
\newcommand{\diag}{\operatorname{diag}}
\newcommand{\argmin}{\operatorname{argmin}}
\newcommand{\Cov}{\operatorname{Cov}}

\begin{document}
\twocolumn[
\aistatstitle{Optimal Multitask Representations Under a Rate Constraint:\\An Information Bottleneck Analysis}
\aistatsauthor{Nicholas Stiles \And  Gregory W. Wornell}
\aistatsaddress{Massachusetts Institute of Technology}
]
\begin{abstract}
Which pretraining objective yields the best representation for downstream tasks under an information budget? We formalize this problem as a multitask information bottleneck, using a mutual information constraint to model incomplete information retention. For Gaussian inputs and tasks, we characterize optimal representations and show that a task-aligned specialist can outperform a reconstruction generalist at low rates but lose at high rates. We also construct task families with increasing SNR for which every fixed Gaussian prediction objective incurs arbitrarily large relative regret, driven by reversals in optimal allocation priorities. Conversely, the average task information matrix defines an objective with a uniform relative-regret bound controlled by maximum task SNR. CelebA and dSprites experiments illustrate that compression can change which pretraining objective performs best, and even task-aware training can reverse which downstream task is better preserved as rate increases.
\end{abstract}

\section{Introduction}
\label{sec:intro}
Representation learning seeks features that support predictions beyond the
objective used to learn them. Image--text pretraining and masked reconstruction have demonstrated broad visual transfer \citep{radford2021transferable,he2022masked}. Theoretical accounts explain how auxiliary prediction can expose latent structure to simple downstream predictors \citep{Arora2019ATA,lee2021,tosh2021}, and how these procedures can produce approximately sufficient representations for related tasks. These results motivate a sufficiency intuition: an objective that requires broadly useful information can supply the information that many tasks need. Under a restricted representation budget, however, preserving a fully sufficient representation may be impossible. This leaves another question:
\emph{which information should receive the available capacity?}

This question concerns both learning compact representations and compressing
features supplied by a pretrained model. An unsupervised ``generalist'' objective may reward preserving many factors in the input, whereas a supervised ``specialist'' emphasizes a smaller set of predictions. 
With unlimited capacity, it is unsurprising that the former objective might yield preferable representations. But when capacity is scarce, the specialist might be preferable even if it omits some information relevant to certain downstream tasks. We study how this tradeoff affects the selection of the representation objective, or {\em pretext}, as a function of representation bitrate. This budget constrains information retained about each observation, rather than network parameter count or computational cost.

We use the term {\em pretext} for any surrogate objective, in place of the downstream/evaluation tasks, which is used to obtain representations. These objectives could be self-supervised in the usual sense \citep{10.1109/ICCV.2015.167, Kolesnikov_2019_CVPR, conf/iclr/GidarisSK18} or supervised on auxiliary labels, distinct from the evaluation labels. Our results concern the information a surrogate prioritizes rather than how its targets are obtained.

Our starting point is task-aware rate-distortion theory \citep{berger1971,cover2006}. Statistical sufficiency concerns preserving the
information required for inference \citep{kullback1951}, while the information
bottleneck (IB) generalizes this condition by trading off relevance against compression \citep{tishby1999}. With unrestricted probabilistic prediction, excess log-loss equals information lost about the target, linking IB to remote source coding under logarithmic loss
\citep{zaidi2020,courtade2014}. Multivariate IB already accommodates several
relevance variables \citep{friedman2001}. We use this established framework to
measure representation quality by task-averaged excess log-loss and capacity
by mutual information.

Several related lines of work clarify what information a pretext should preserve. Task-aware compression can exploit shared invariances to control predictive loss across an entire task family \citep{dubois2021compression}. View-selection principles similarly seek to discard shared nuisance information while retaining task relevance \citep{tian2020}. Spectral and kernel analyses connect generalization to augmentation structure, representation dimension, and inductive biases \citep{haochen2021,pmlr-v202-cabannes23a}. Our question is complementary to these analyses: for a certain task suite and information budget, does optimizing a
surrogate prediction objective allocate capacity appropriately, and can the
same surrogate remain appropriate across budgets? This population-level allocation question differs from the statistical benefit of learning shared structure from multiple tasks \citep{maurer2016}, and also abstracts from decoder restrictions, whose importance is formalized by the decodable IB
\citep{dubois2020}.

Gaussian task models make the allocation question explicit. Classical
Gaussian IB (GIB) yields noisy linear representations with a fixed canonical basis and rate-dependent mode activation \citep{JMLR:v6:chechik05a}. Building on Gaussian extremality results \citep{liu2017informationtheoreticperspectivesbrascampliebinequality}, we reduce the multitask Gaussian problem to optimization of a conditional covariance. We then compare its solutions with those induced by single-task Gaussian pretexts. Even when all task matrices share an eigenbasis, their optimal allocation ordering can reverse with rate. Directions serving frequent but less precise tasks can take priority at low rates, while directions serving rare yet more precise tasks become more important at high rates. This is stronger than sequential activation of additional modes within a fixed ordering, as occurs in the usual GIB.

We establish both negative and positive consequences. We show that weakly task-aligned specialist objectives can outperform a generalist at low rates but lose at higher rates due to omitted information. More strongly, we show that some task families force substantial relative regret for any rate-independent single-target Gaussian pretext. Conversely, bounded task SNR guarantees a rate-uniform approximation for a certain pretext constructed using the average task information matrix, and its regret bound is second order at low SNR.

Experiments on frozen CelebA features compare compression objectives at matched budgets, using both the exact Gaussian solution and trained finite-bit encoders, and exhibit rate-dependent pretext rankings and specialist–generalist crossovers. A second experiment trains encoders end-to-end on a two-factor dSprites mixture and finds the predicted reversal in relative factor recoverability, near the rate the Gaussian analysis identifies. Our results therefore identify when the deployment budget matters to surrogate-objective selection, and when even a task-aware fixed surrogate may be insufficient to approximate the optimal allocation at all rates.

\section{From Single-Task to Multitask Information Bottleneck}
\label{sec:setup}

\subsection{Single-Task IB: What Does a Pretext Preserve?}
\label{sec:single-task-ib}
For input $X$ and prediction target $V$, the information bottleneck
\citep{tishby1999} chooses an encoder $P_{Z\mid X}$ satisfying
$Z-X-V$ to solve
\begin{equation}
    \max_{I(X;Z)\le R} I(Z;V)
    \quad\Longleftrightarrow\quad
    \min_{I(X;Z)\le R} I(X;V\mid Z).
    \label{eq:single-task-ib}
\end{equation}
All information quantities are in bits. The target specifies which
information is useful, while $R$ limits how much the representation retains.
With an unrestricted probabilistic decoder, the minimized quantity is the
increase in optimal predictive log-loss caused by replacing $X$ with $Z$.
The budget is an information-theoretic model of representation compression.

The Gaussian case makes this allocation explicit. For details, see \citep{JMLR:v6:chechik05a}. Let $X\sim\N(0,I_d)$ and
$V=AX+\epsilon$, with independent $\epsilon\sim\N(0,N)$ and $N\succ0$, and define the \textit{task information matrix}
$H=A^\top N^{-1}A=U\diag(h_i)U^\top$. The eigenvalues of $H$ describe task signal-to-noise ratios along feature
directions. If $H \succ 0$, then an optimal Gaussian
representation allocates
\begin{align}
    r_i^H(R)&=\left[\tfrac12\log_2 h_i-\tau_R\right]_+,
    &\sum_i r_i^H(R)&=R,\nonumber\\
    K_R^H&=U\diag(2^{-2r_i^H(R)})U^\top,
    \label{eq:pretextalloc}
\end{align}
where $K_R^H=\Cov(X\mid Z)$ is the residual covariance, and ``water level'' $\tau_R$ is chosen to satisfy the rate constraint. Observe that more bits (greater $r_i^H$) reduce uncertainty in the corresponding direction. Crucially, a fixed Gaussian target fixes both the basis and the ordering of the allocations: if $h_i\ge h_j$, then $r_i^H(R)\ge r_j^H(R)$ at every rate. Increasing the budget changes the allocation, but cannot reverse this ordering. We next ask whether this structure suffices when the representation must support multiple tasks.

\subsection{A Shared Representation for Multiple Tasks}
Let $t\in\{1,\ldots,m\}$ index evaluation tasks, with weights $\pi_t>0$
and $\sum_t\pi_t=1$. The shared encoder does not observe $t$, while a
task-dependent decoder does, mirroring the case where a pretrained representation is adapted to an evaluation task through probing. Under $Z-X-Y_t$, define distortion measures
\begin{align}
    D_t(Z)&=\E\operatorname{KL}(P_{Y_t\mid X}\|P_{Y_t\mid Z})
           =I(X;Y_t\mid Z),\label{eq:distortion}\\
    D_\pi(Z)&=\sum_t\pi_tD_t(Z),\nonumber\\
    D_\pi^*(R)&=\inf_{I(X;Z)\le R}D_\pi(Z).
\end{align}
Thus $D_\pi^*(R)$ is the smallest average compression-induced excess
log-loss available at rate $R$. For a pretext target $V_a$, let
\begin{align}
    Z_{a,R}&\in\arg\min_{I(X;Z)\le R}I(X;V_a\mid Z),\nonumber\\
    \Reg_\pi(a,R)&=D_\pi(Z_{a,R})-D_\pi^*(R).
    \label{eq:regret}
\end{align}
A fixed pretext fixes the target channel $P_{V_a\mid X}$, but its encoder is re-optimized at each rate. $\Reg_\pi(a,R)$ measures the excess loss of substituting a surrogate IB with target $V_a$ in place of the actual multitask IB defined by the evaluation tasks.

For $T\sim\pi$ independent of $X$, the task-indexed target
$Y^\dagger=(T,Y_T)$ satisfies
$I(Z;Y^\dagger)=\sum_t\pi_tI(Z;Y_t)$. The multitask objective is
therefore itself an IB objective, but $Y^\dagger$ is generally not jointly
Gaussian with $X$, even when every individual task is Gaussian.

\subsection{Gaussian Task Model and Exact Benchmark}
\label{sec:gaussian-task-model}
Assume
\begin{equation}
    X\sim\N(0,I_d),\qquad Y_t=A_tX+\epsilon_t,
    \label{eq:gaussian-task-model}
\end{equation}
where $\epsilon_t\sim\N(0,N_t)$ is independent of $X$ and
$N_t\succ0$. No coupling between task noises is needed. Define
\begin{align}
    H_t&=A_t^\top N_t^{-1}A_t,
    &C_t&=H_t(I+H_t)^{-1},\label{eq:task-information-matrices}\\
    \overline H&=\sum_t\pi_tH_t,
    &\overline C&=\sum_t\pi_tC_t.
    \label{eq:average-information-matrices}
\end{align}
We assume $\overline H\succ0$, but otherwise our results hold when restricting to the task-relevant subspace $\operatorname{span}_t\operatorname{range}(H_t)$.

Our first result shows that the multitask GIB is solved by a representation which is jointly Gaussian with $X$.
\begin{proposition}[Gaussian reduction]
\label{prop:gaussian}
Under \eqref{eq:gaussian-task-model} and $\overline H\succ0$, for every
finite $R\ge0$,
\begin{align}
    D_\pi^*(R)&=\min_{K\in\mathcal K_R}D_\pi(K),\label{eq:program}\\
    D_\pi(K)&=\frac12\sum_t\pi_t
       \log_2\det(I+H_t^{1/2}KH_t^{1/2}),
       \label{eq:gaussian-covariance-objective}\\
    \mathcal K_R&=\{0\prec K\preceq I:
       -\tfrac12\log_2\det K\le R\}.
       \label{eq:gaussian-feasible-set}
\end{align}
The minimizer $K_R$ exists, is unique, and exhausts the rate budget:
\begin{equation}
    -\tfrac12\log_2\det K_R=R.
    \label{eq:gaussian-rate-exhaustion}
\end{equation}
It is realized by
\begin{equation}
    Z=(K_R^{-1}-I)^{1/2}X+\xi,\qquad
    \xi\sim\N(0,I_d),\quad \xi\perp X,
    \label{eq:gaussian-optimal-encoder}
\end{equation}
for which $\Cov(X\mid Z)=K_R$ and $I(X;Z)=R$.
\end{proposition}
Gaussian extremality
\citep[Theorem~15]{liu2017informationtheoreticperspectivesbrascampliebinequality}
reduces optimization over arbitrary encoders to this covariance program. The proof and optimality conditions are in Appendix~\ref{app:baseline}.
This supplies an exact benchmark in the model, so failures of a pretext can be attributed to the information it prioritizes, through $K_R$.

\section{Task Knowledge and Pretext Choice}
\label{sec:taskknowledge}
We next consider how weak task knowledge informs how one should spend a limited representation budget. Consider two types of objectives: one concentrates on a strict subspace (a ``specialist''), while the other preserves information about the entire input (a ``generalist''). For example, the former could include a suite of supervised attribute prediction tasks, while the latter could include broader reconstructive, contrastive, or JEPA-style objectives \citep{assran2023}.

Throughout, $X\sim\N(0,I_d)$, $D_\pi$ denotes the evaluation distortion from Proposition~\ref{prop:gaussian}, and $\overline C$ is as defined in (\ref{eq:average-information-matrices}). $\overline C$ measures the sensitivity of distortion to rate allocation near zero rate.

\subsection{When Specialists Beat a Generalist}
\label{sec:crossovers}
Let $S$ be a $k$-dimensional subspace, $1\le k<d$, with orthonormal
basis $u_1,\ldots,u_k$ and projector $P$. Consider equally weighted
specialist pretexts
\begin{equation}
V_i=\sqrt h\,u_i^\top X+\epsilon_i,\qquad
\epsilon_i\sim\N(0,1),\quad h>0,
\label{eq:specialistpretexts}
\end{equation}
whose noises are independent of $X$. This defines a multitask GIB problem which the specialist solves by allocating $R/k$ bits to each direction in $S$.

Next, let the generalist use the pretext $V_g=X+\eta$ with independent
$\eta\sim\N(0,\sigma_g^2I)$, $\sigma_g^2>0$.
Its optimal allocation is isotropic: each direction receives $R/d$ bits. The resulting residual covariances are
\begin{align}
K_s(R)&=2^{-2R/k}P+(I-P),\nonumber\\
K_g(R)&=2^{-2R/d}I.
\label{eq:focuscoverage}
\end{align}
The following proposition establishes that one can expect a crossover in the relative performance of the two objectives as rate increases.
\begin{proposition}[Specialist--generalist crossover]
\label{prop:crossovers}
Suppose the specialist span $S$ has above-average evaluation relevance per
dimension:
\begin{equation}
\Delta=\frac{\tr(P\overline C)}{k}
-\frac{\tr(\overline C)}{d}>0.
\label{eq:specialistalignment}
\end{equation}
Suppose also that it omits evaluation-relevant information, so
\begin{equation}
D_{\rm miss}:=D_\pi(I-P)>0.
\label{eq:missingdistortion}
\end{equation}
Then $D_\pi(K_s(R))<D_\pi(K_g(R))$ for all sufficiently small
$R>0$, whereas the reverse strict inequality holds for all sufficiently
large $R$. In particular, the curves cross at a finite positive rate.
\end{proposition}

The mechanism follows from their initial slopes and limiting losses:
\begin{align}
D_\pi(K_g(R))-D_\pi(K_s(R))
&=\Delta R+O(R^2),\nonumber\\
D_\pi(K_s(R))&\longrightarrow D_{\rm miss},\nonumber\\
D_\pi(K_g(R))&\longrightarrow0.
\label{eq:crossovermechanism}
\end{align}
Specialists obtain more downstream benefit per bit initially, but their
limited coverage creates a loss floor. Condition~\eqref{eq:specialistalignment} relates pretraining directions to evaluation relevance, but actual pretext and
evaluation targets need not coincide. Condition~\eqref{eq:missingdistortion}
is equivalent to $\tr((I-P)\overline H)>0$. We illustrate Proposition \ref{prop:crossovers} in Figure \ref{fig:specialist_vs_generalist_synth}, and its full proof is in Appendix~\ref{app:knowledgeproofs}.

\begin{figure}[t]
    \centering
    \includegraphics[width=0.7\linewidth]{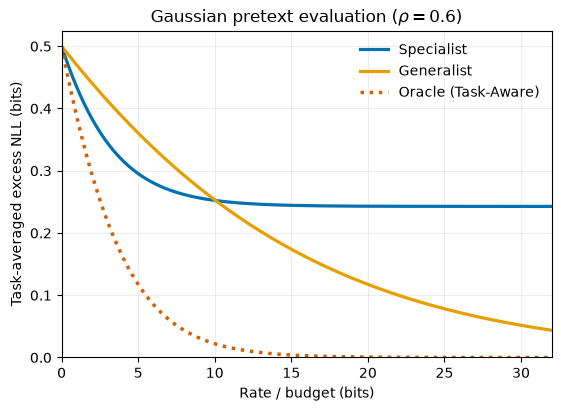}
    \caption{Illustration of Proposition \ref{prop:crossovers}. Here, $k=4$, $d=16$, and the evaluation tasks span a $k$-dimensional subspace for which squared cosines of the principal angles with $S$ all equal $\rho = 0.6$. SNR of all tasks equals 1, including the specialist tasks, reconstruction, and the evaluation tasks.}
\label{fig:specialist_vs_generalist_synth}
\end{figure}

This crossover demonstrates that transferability of certain pretexts is rate-dependent, but does not establish that every fixed Gaussian pretext must lead to suboptimal representations for the multitask problem. We address the stronger question next.

\section{Reversals in Optimal Allocation Priorities}
\label{sec:obstruction}
Increasing the information budget can change the relative priority of input factors, even when the task suite remains fixed. Directions serving frequent tasks can receive more rate initially, while directions serving rare but more precise tasks can receive more at higher budgets. We first exhibit this reversal and quantify its consequences for fixed Gaussian pretexts, whose allocation ordering cannot change with rate (Section \ref{sec:single-task-ib}). We then characterize a rate-dependent pretext that matches the multitask optimum.

\subsection{Example: task frequency versus precision}
\label{sec:obstruction_example}
We can illustrate the obstruction with a simple two-task example: one task is common and requires moderate precision, while the other is rare but supports much finer prediction. For $M > 1$, define rank-one task matrices
\begin{multline}
    H_1=M^2e_1e_1^\top,\quad H_2=M^4e_2e_2^\top,\\
(\pi_1,\pi_2)=(1-M^{-1},M^{-1}).
\label{eq:diagonalexample}
\end{multline}
Their optimal covariance is diagonal. Writing $p=1-M^{-1}$ and $q=M^{-1}$,
its distortion is the minimum over $0\le r\le R$ of
\begin{multline}
F_R(r)=\tfrac p2\log_2(1+M^2 2^{-2r}) \\
+\tfrac q2\log_2(1+M^4 2^{-2(R-r)}).
\label{eq:diagonalf}
\end{multline}
For $M>3$, at $R_1=\log_2 M$ all rate goes to mode~1: its marginal
benefit at $r=R_1$ is $p/2>q$, the maximum marginal benefit of mode~2.
At $R_2=4\log_2 M$, mode~2 instead receives more rate. At equal allocation
its marginal benefit is $q/2$, exceeding $p/(1+M^2)$ for mode~1. See Appendix~\ref{app:gap} for this section's calculations.

\begin{figure}[t]
    \centering
    \includegraphics[width=0.75\linewidth]{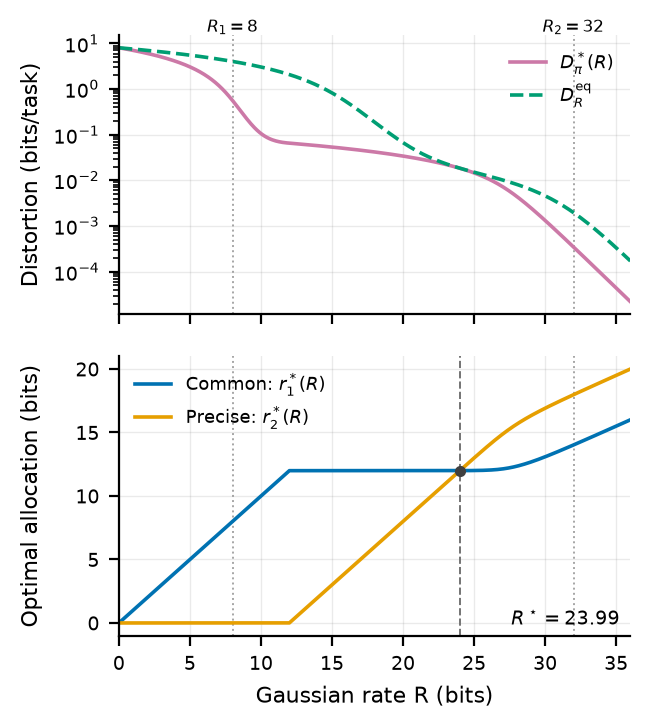}
    \caption{Rate--distortion gap (top) and optimal rate allocations (bottom) for the two-task example with $M=256$. Dotted vertical lines mark $R_1=8$ and $R_2=32$, and the dashed line marks the exact reversal $R^* \approx 24$.}
    \label{fig:reversal}
\end{figure}

The mechanism is \textit{task frequency} versus \textit{precision}. At low rate, a direction's marginal value saturates at its task weight. At high rate, residual signal strength matters: here $qM^4>pM^2$. Neither allocation decreases with total rate, and the active supports remain nested. What reverses is their \emph{relative priority}, which does not happen in standard single-target GIB.

We can also compute the relative suboptimality of approximating the multitask GIB's solutions by those of a single-target GIB. Let $D_R^{\rm eq}=F_R(R/2)$ and $g_R=D_R^{\rm eq}-D_\pi^*(R)$. For this two-dimensional family, every fixed Gaussian pretext, including one in a different basis, has the wrong coordinate ordering at one of the two rates. Its distortion there is at least $D_R^{\rm eq}$, so
\begin{equation}
\inf_{H\succ0}\max_{j=1,2}
\frac{\Reg_\pi(H,R_j)}{D_\pi^*(R_j)}
\ge \min_{j=1,2}\frac{g_{R_j}}{D_\pi^*(R_j)}.
\label{eq:examplerelative}
\end{equation}
For $M=256$, we obtain a lower bound exceeding $4.67$, meaning every fixed Gaussian pretext has distortion exceeding
$5.67$ times the optimum at at least one of the rates $R_1 = 8$ or $R_2 = 32$. This gap is visualized in Figure \ref{fig:reversal}.

More generally, for commuting tasks, the marginal value of a bit at rate $r$ is
\begin{equation}
m_i(r)=\sum_t\pi_t\frac{h_{ti}2^{-2r}}{1+h_{ti}2^{-2r}}.
\label{eq:marginal}
\end{equation}
These curves can cross, while for one Gaussian
target, their ordering is fixed by its eigenvalues. The example shows that this obstruction can entail a large relative loss for certain standard choices of pretext.

\subsection{Rate-Dependence of Optimal Pretext Geometry}
While a rate-independent optimal Gaussian pretext need not exist, a rate-dependent one always does. That is, at a given rate $R$, there exists a single-task GIB problem sharing the same solution as the multitask problem. This suggests that multitask pretraining may be unnecessary if the objective is designed with deployment bitrate in mind.

Here, we consider Gaussian pretexts defined by additive corruption,
\begin{equation}
    V_a=X+\eta_a,
    \qquad
    \eta_a\sim\N(0,\Sigma_a),
    \qquad
    \eta_a\perp X,
\end{equation}
where $\Sigma_a\succ0$. Choosing a pretext amounts to choosing the
directions and strengths of corruption. Its information matrix is
$H_a=\Sigma_a^{-1}$.

The following result establishes that for any task family $\{H_t,\pi_t\}$ and rate $R$, the corresponding IB problem shares the same solutions as a certain single-task IB problem with a covariance depending on $K_R$. Refer to Section \ref{sec:setup} for definitions and to Appendix~\ref{app:matched} for the proof.

\begin{theorem}[Downstream-matched pretext]
\label{thm:matched}
Under the assumptions of Proposition~\ref{prop:gaussian}, suppose that
$\overline H\succ0$. For $R>0$, define
\begin{align}
    G_R
    &=
    \sum_t\pi_t
    \left[H_t - H_t(K_R^{-1}+H_t)^{-1}H_t \right],
    \label{eq:GR}\\
    \Sigma_a^*(R)
    &=
    G_R^{-1}-K_R.
    \label{eq:matched}
\end{align}
Then
\begin{equation}
    0\prec G_R\prec K_R^{-1},
    \qquad
    \Sigma_a^*(R)\succ0.
\end{equation}
Moreover, $K_R$ is a global optimum of the rate-constrained Gaussian
pretext problem
\begin{multline}
    K_R\in
    \argmin_{K\in\mathcal K_R}
    \frac12\log_2\det\!\Big(
        I+\\ \Sigma_a^*(R)^{-1/2}
        K
        \Sigma_a^*(R)^{-1/2}
    \Big).
\end{multline}
\end{theorem}
Many Gaussian pretexts can induce a prescribed covariance (e.g., by scaling the noise by a constant), so the statement is not about uniqueness. The takeaway from Theorem~\ref{thm:matched} is the matching rule \eqref{eq:matched} and its dependence on rate.

Specializing Theorem \ref{thm:matched} to low- and high- rate limits illustrates how the optimal pretext geometry changes with rate. Recalling the saturated task matrix $\overline C = \sum_t \pi_t C_t$ from (\ref{eq:average-information-matrices}), if $\overline C$ has a simple largest eigenvalue with unit eigenvector $v$, then
\begin{equation}
    \frac{I-K_R}{2(\ln2)R}
    \longrightarrow vv^\top
    \qquad\text{as }R\downarrow0.
    \label{eq:low}
\end{equation}
On the other hand, at high rate,
\begin{equation}
    \frac{K_R}{c_R}
    \longrightarrow \overline H^{-1},
    \quad
    c_R=(\det\overline H)^{1/d}2^{-2R/d},
    \quad R\to\infty.
    \label{eq:high}
\end{equation}
These endpoint expansions, calculated in Appendix \ref{app:endpoint-expansions}, determine the limiting matched pretexts.

\begin{corollary}[Limiting matched pretexts]
\label{cor:limits}
Under the assumptions of Theorem~\ref{thm:matched},
\begin{equation}
    \Sigma_a^*(R)
    \longrightarrow
    \begin{cases}
        \overline C^{-1}-I,
        & R\downarrow0,\\[3pt]
        \overline H^{-1},
        & R\to\infty.
    \end{cases}
    \label{eq:matched-limits}
\end{equation}
\end{corollary}

At low rate, the marginal value of a direction is governed by
\[
    \overline C
    =
    \sum_t\pi_tH_t(I+H_t)^{-1}.
\]
The transformation $H_t\mapsto H_t(I+H_t)^{-1}$ saturates strong task
signals, so making an already predictable task still more precise has
limited effect on the first allocated bits. At high rate, the remaining
loss is instead
\[
    D_\pi(K)
    =
    \frac{\tr(\overline HK)}{2\ln2}
    +o(\|K\|),
\]
and task precision enters linearly. Consequently, the appropriate
pretext can change the relative priority of its directions across
budgets.

\section{When a Fixed Pretext Suffices}
\label{sec:bounded-snr}
The preceding example shows that a fixed Gaussian pretext can incur large
relative regret, but uses tasks with widely separated, high signal-to-noise
ratios (SNRs). We now give a complementary positive result: bounded task SNR
limits the excess distortion of a certain fixed pretext, uniformly over all rates.

Define the Specht ratio
\begin{equation}
S(\kappa)=\frac{(\kappa-1)\kappa^{1/(\kappa-1)}}{e\ln\kappa},
\qquad \kappa>1,\qquad S(1)=1.
\label{eq:specht-ratio}
\end{equation}
\begin{theorem}[Bounded-SNR fixed-pretext approximation]
\label{thm:bounded-snr}
Under Proposition~\ref{prop:gaussian}, suppose
$0\preceq H_t\preceq LI$ for every task and
$\overline H=\sum_t\pi_tH_t\succ0$.
Let $K_R^{\overline H}$ denote the rate-$R$ optimal covariance for the
single Gaussian pretext with information matrix $\overline H$.
For every finite $R\ge0$,
\begin{equation}
D_\pi(K_R^{\overline H})
\le S(1+L)D_\pi^*(R).
\label{eq:bounded-snr-factor}
\end{equation}
Consequently,
\begin{equation}
\inf_{H\succ0}\sup_{0\le R<\infty}
\frac{\Reg_\pi(H,R)}{D_\pi^*(R)}
\le S(1+L)-1.
\label{eq:bounded-snr-regret}
\end{equation}
\end{theorem}

The pretext is fixed across rates, but requires knowledge of the task suite, as it depends on the weighted average task matrix $\overline H$. Recalling Corollary \ref{cor:limits}, we see that Theorem \ref{thm:bounded-snr} bounds the cost of approximating the multitask IB solution by its high-rate limit.

The bound is independent of task count and input dimension, and applies
even when optimal directions rotate. In particular, $L=2$ gives
$S(3)\approx 1.160$, bounding regret of the optimal fixed Gaussian pretext by $16\%$. At low SNR,
\begin{equation}
S(1+L)-1=\frac{L^2}{8}+O(L^3).
\label{eq:bounded-snr-small-L}
\end{equation}
Thus the guaranteed relative loss is second order in task strength.

See Appendix~\ref{app:bounded-snr} for the proof of Theorem \ref{thm:bounded-snr}. Intuitively, weak Gaussian task channels
have nearly linear distortion in $K$, so their objectives approximately
aggregate into one task. Strong channels can depart substantially from
this approximation through saturation.

\section{Experiments}
\label{sec:experiments}
Our experiments test two claims. First, whether compression changes the ranking of pretext objectives on held-out tasks as predicted in Section \ref{sec:taskknowledge}, and second, whether direct multitask training exhibits the allocation reversals predicted in Section \ref{sec:obstruction}. Code for all experiments is available at our \href{https://github.com/ngs123/multitask_information_bottleneck}
{\textcolor{blue}{GitHub repository}}.

\subsection{Pretext Selection Under Compression}
\label{sec:experiments-selection}

\paragraph{Setup.} We use CelebA \citep{liu2015deeplearningfaceattributes} with its official train/validation/test partitions. A frozen ImageNet-pretrained ResNet-18 \citep{he2015deepresiduallearningimage} produces 512-dimensional features, reduced to $d=32$ by PCA and whitening fitted on training data. These $d$-dimensional features define the common input $X$ to the IB problems. All objectives thus receive this same set of uncompressed ``master features", with compression heads re-optimized at each budget.

The pretext bank includes supervised specialists predicting various attribute-groups, feature reconstruction, augmented-view and complementary-view prediction, and semantic--reconstruction mixtures. Our prespecified primary comparison pairs feature reconstruction under Gaussian NLL with a specialist predicting the five attributes \textit{Mouth Slightly Open, High Cheekbones, Wavy Hair, Straight Hair}, and \textit{Receding Hairline}. Evaluation equally averages the four disjoint attributes \textit{Smiling, Eyeglasses, Blond Hair}, and \textit{Young}. We fixed the specialist and evaluation attribute subsets before inspecting the rate–distortion curves and did not search over alternative subsets to obtain a crossover. Weak correlations between the disjoint pretext and evaluation attributes motivate this comparison as an empirical analogue of the task alignment studied in Proposition \ref{prop:crossovers}. We use four evaluation attributes rather than all forty provided by the dataset to keep the rate sweep tractable.

\paragraph{Compression mechanisms and evaluation.}
We compare two compression mechanisms: the direct Gaussian covariance optimization of Proposition \ref{prop:gaussian} under models fitted to task/feature statistics, and a learned $B$-bit discrete bottleneck trained on the empirical pretext loss. We directly optimize the covariance program (\ref{eq:gaussian-covariance-objective}) in the former case, and use the straight-through gradient estimator in the latter (training details are given in Appendix \ref{app:experiment-details}). Their respective rate axes are the nominal Gaussian rate $\frac12\log_2\det(I+L\Sigma_XL^\top)$ for the channel $Z = LX + \xi$, and code budget $B$. Both upper-bound retained information.

To proxy distortion $D_\pi$, we measure task-averaged excess negative log-likelihood (NLL) relative to probes trained on uncompressed features, using the same probe class and training protocol across objectives and rates. This estimates the predictive cost of compression, subject to probe approximation and estimation error. We report 95\% pointwise percentile bootstrap intervals from 2,000 paired resamples. Gaussian-channel intervals resample test images, while learned-code intervals independently resample test images and the five fitted training seeds. Appendix \ref{app:celeba-training-details} details the uncertainty estimates, crossover estimation, and code entropies.

\paragraph{Results.} In Figure \ref{fig:celeba}, we see that in both mechanisms, the specialist performs better at low budgets and the generalist overtakes it as the budget grows. Crossovers occur near $16.8$ bits for the Gaussian channel and $21.2$ bits for discrete codes, with a sign change in 100\% of resamples. The specialist curve flattens while the generalist continues improving, consistent with (\ref{eq:crossovermechanism}). The fitted proxies satisfy both conditions of Proposition \ref{prop:crossovers}: above-average evaluation relevance per specialist dimension, $\Delta=0.018>0$, and omitted evaluation information, $D_{\mathrm{miss}}=0.215>0$.

\begin{figure}[t]
    \centering
    \includegraphics[width=\linewidth]{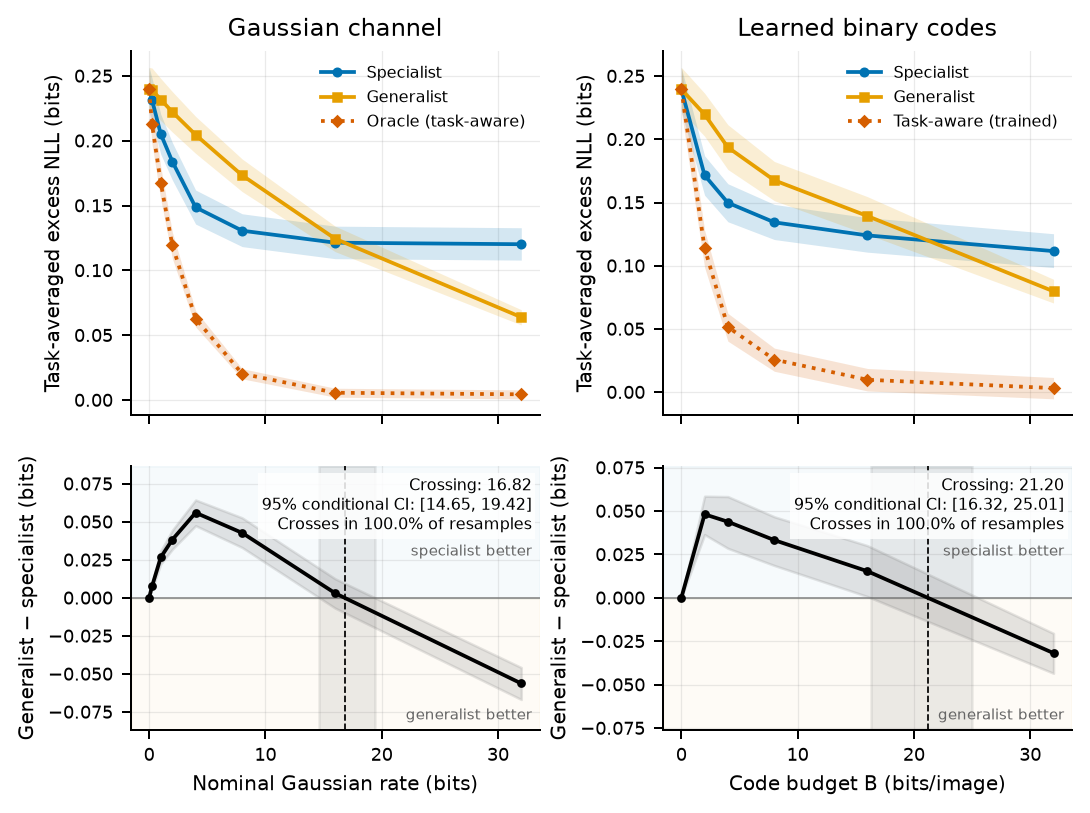}
    \caption{CelebA specialist--generalist comparison. Left: Gaussian channels under fitted task models, against nominal Gaussian rate. Right: learned discrete codes, against code budget. Top: task-averaged excess NLL. Bottom: generalist minus specialist distortion. Shading shows 95\% pointwise percentile bootstrap intervals from test-image resampling (left) and crossed resampling of test images and training seeds (right).}
    \label{fig:celeba}
\end{figure}

Across the full bank, 9 of 28 pretext pairs reverse order between the lowest and highest nonzero Gaussian budgets, and 13 of 28 reverse for discrete codes; corresponding Kendall correlations are $\tau_b=0.357$ and $0.07$. The primary crossover also persists with a one-hidden-layer MLP probe, supporting robustness beyond linear probes. Appendix \ref{app:celeba-training-details} reports full rankings, probe sensitivity, and the relationship between Gaussian proxy predictions and measured binary NLL.

\subsection{Allocation Reversal Under Direct Multitask Training}
\label{sec:experiments-allocation}

\paragraph{Setup.} We now fix the training objective to the evaluation mixture and test whether relative factor recoverability reverses with rate. This tests the hypothesis of Section \ref{sec:obstruction} that the multitask problem does not, in general, admit a rate-independent ordering of priorities across different information sources.

Using dSprites \citep{dsprites17}, we instantiate an analogous two-mode construction to that of Section \ref{sec:obstruction_example} with $M=4$. Prediction of horizontal position is a common and moderately precise task, while vertical position prediction is rare and highly precise, with weights $(0.75,0.25)$ and task SNRs $(16,256)$. We tune the SNRs by adding an appropriate amount of Gaussian noise. A convolutional encoder is then trained end-to-end on this mixture using a variational bottleneck, sweeping the compression penalty $\beta$ with three seeds per setting.

\paragraph{Rate and recoverability.} For encoder $q_\phi(z \mid x)$ and reference measure $p_0(z)$,
\begin{multline}
    \mathcal{R}_{\text{VIB}} = \mathbb{E}_X D_{\mathrm{KL}}\!\left(q_\phi(Z \mid X)\,\|\,p_0(Z)\right) \\= I(X;Z) + D_{\mathrm{KL}}\!\left(q_\phi(Z)\,\|\,p_0(Z)\right),
\end{multline}
so the achieved KL is an upper bound on retained information. In Appendix \ref{app:dsprites-training-details}, we plot achieved $\mathcal{R}_{\text{VIB}}$ in bits per image and report it per seed because achieved rate varies at fixed $\beta$.

Next, factor targets are standardized to unit variance,  and for each factor we report the \textit{recoverability index}
\begin{equation}
    \mathcal{I}_i = -\tfrac12\log_2 \mathrm{MSE}_i
\end{equation}
for a held-out probe, using an identical probe protocol at every budget. In the Gaussian model $\mathcal{I}_i$ equals the rate allocated to direction $i$, while for general factors it measures effective bits recovered.

\paragraph{Results.} Figure \ref{fig:dsprites} illustrates that both factors become more recoverable as rate increases, but their ordering reverses near $\hat R_\times=5.3$ bits, individually in all three seeds. This is the empirical analogue of Figure \ref{fig:reversal}, where $r_1^*$ and $r_2^*$ are both nondecreasing while their relative priority inverts.

The Gaussian construction of Section \ref{sec:obstruction_example} predicts a reversal at $R^\star\approx5.30$ bits. This agreement supports the proposed mechanism, but is not a calibrated rate prediction: $\mathcal{R}_{\text{VIB}}$ bounds rather than equals $I(X;Z)$, it can include nuisance factors, and $\mathcal{I}_i$ is exact only under a Gaussian model. Exchanging the factors' task roles reverses the effect (Appendix \ref{app:dsprites-training-details}), supporting the frequency-precision explanation over an intrinsic asymmetry between coordinates.

\paragraph{Scope.} The theory concerns globally optimal Gaussian multitask representations, which can reverse their allocation ordering with rate. The experiment shows that representations obtained by directly optimizing a variational multitask bottleneck from images exhibit the analogous reversal in relative factor recoverability, near the predicted rate. However, it does not certify optimality, so we describe these as task-optimized neural representations, not as empirically optimal ones. The theorem supplies the statement about optimal representations, while the experiment shows that its qualitative mechanism survives neural learning from non-Gaussian (in fact, discrete) observations.

\begin{figure}[t]
    \centering
    \includegraphics[width=0.75\linewidth]{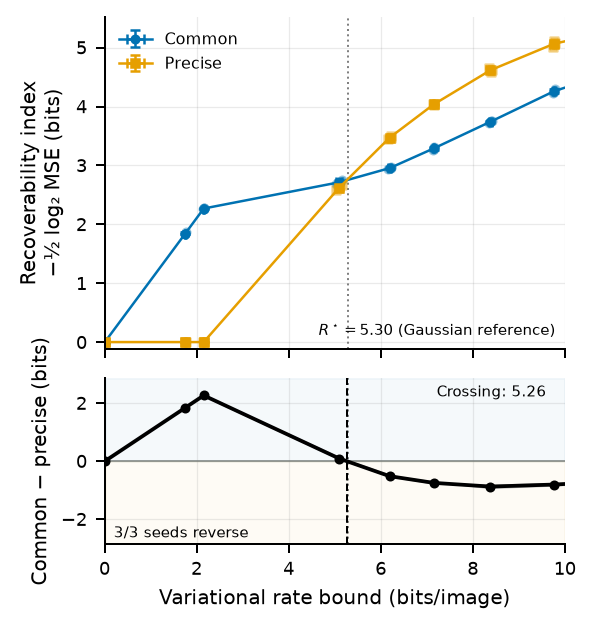}
    \caption{Factor recoverability reversal under direct multitask training on dSprites, $M=4$. Top: common/moderate-precision and rare/high-precision factors' recoverability against achieved variational rate. Points show individual seeds, vertical bars indicate recoverability spread, and horizontal bars show rate spread at fixed $\beta$. Bottom: difference in recoverability indices. The dashed vertical line marks the Gaussian prediction $R^\star\approx5.30$ bits, while the observed crossover is $\hat R_\times\approx5.26$ bits.}
    \label{fig:dsprites}
\end{figure}

\subsection{Experimental Scope}
The two experiments exhibit related but distinct reversals worth separating. Section \ref{sec:experiments-selection} demonstrates reversals within the ranking of objectives: which surrogate to optimize depends on the budget. Section \ref{sec:experiments-allocation} demonstrates reversals in the priority assigned to directions within the oracle objective's solution, showing that which information to keep depends on how much may be kept. Thus, these experiments are complementary demonstrations of rate-dependent objective rankings and rate-dependent allocation priorities.

Both experiments are limited in their own ways. Section \ref{sec:experiments-selection} fixes a backbone, so strictly it tests compression-objective selection rather than learning a complete feature extractor per budget. Section \ref{sec:experiments-allocation} uses a synthetic dataset with a deliberately constructed noisy task pair, so it demonstrates that the mechanism is realizable under neural learning, not that it is necessarily common in natural task suites.

\section{Concluding Remarks}
\label{sec:concluding-remarks}
Our results motivate evaluating representation objectives across the intended deployment bitrate range. The most direct applications concern representations compressed for transmission or storage, where the information budget is an explicit design choice. An objective that uses scarce capacity effectively may become limited by the information it omits at higher rates.

Our analysis also identifies when adapting the objective to the budget may matter most. Task frequency and precision jointly determine allocation priorities: rare tasks requiring precise predictions can receive greater priority as rate grows. In the Gaussian model, bounded task SNR limits the cost of using a rate-independent, task-aware objective, and the average task information matrix provides a principled choice. This guarantee requires knowledge of the evaluation suite; it does not establish that a generic reconstruction objective suffices. The experiments illustrate these allocation effects, while leaving their prevalence in broader task suites unresolved.

An open question is whether training limitations and architectural restrictions induce analogous allocation tradeoffs, which requires distinguishing information loss from information inaccessible to restricted decoders. Extending the analysis to weight quantization would additionally require connecting weight precision to the information retained by representations.

\bibliographystyle{plainnat}
\nocite{*} 
\bibliography{bib}

\clearpage
\appendix
\thispagestyle{empty}
\onecolumn
\aistatstitle{Appendix}

\section{Proof of the Gaussian Reduction}
\label{app:baseline}

Throughout this appendix, all logarithms appearing in information
quantities are base two unless otherwise indicated. We define
\begin{equation}
    \rho(K)
    =
    -\frac12\log_2\det K.
    \label{eq:rho-definition}
\end{equation}

\subsection{Gaussian extremality in Lagrangian form}

We use the following established Gaussian extremality result. It is the
only external extremality ingredient in the proof.

\begin{lemma}[Gaussian extremality]
\label{lem:gaussian-extremality}
Let $X\sim\N(0,I_d)$ and let
\[
    Y_t=A_tX+\epsilon_t,
    \qquad
    \epsilon_t\sim\N(0,N_t),
    \qquad
    N_t\succ0.
\]
For every $\beta>0$,
\begin{align}
    &\inf_{P_{Z\mid X}}
    \left\{
        \sum_t\pi_t I(X;Y_t\mid Z)
        +
        \beta I(X;Z)
    \right\}
    \notag\\
    &\qquad=
    \inf_{0\prec K\preceq I}
    \left\{
        \frac12\sum_t\pi_t
        \log_2\det\!\left(
            I+H_t^{1/2}KH_t^{1/2}
        \right)
        -
        \frac{\beta}{2}\log_2\det K
    \right\}.
    \label{eq:gaussian-extremality-lagrangian}
\end{align}
The right-hand side is attained by a jointly Gaussian
representation whose residual covariance is
$K=\Cov(X\mid Z)$.
\end{lemma}

\begin{proof}
This is the specialization of the Gaussian extremality theorem
\citep[Theorem~15]{liu2017informationtheoreticperspectivesbrascampliebinequality} to the weighted collection of linear
Gaussian channels $(A_t,N_t,\pi_t)_t$. Constants independent of the
encoder cancel after writing
\[
    I(X;Y_t\mid Z)
    =
    h(Y_t\mid Z)-h(\epsilon_t), \quad \text{and} \quad
    I(X;Z)=h(X)-h(X\mid Z).
\]
The extremality theorem states that the resulting weighted conditional
entropy functional is minimized by a Gaussian test channel. Such a
channel is parameterized by its residual covariance
$0\prec K\preceq I$, giving
\eqref{eq:gaussian-extremality-lagrangian}.
\end{proof}

\subsection{The Gaussian covariance parameterization}

We first record the rate and distortion induced by an arbitrary
Gaussian residual covariance.

\begin{lemma}[Gaussian realization]
\label{lem:gaussian-realization}
For every $0\prec K\preceq I$, let $B=(K^{-1}-I)^{1/2}$ and define
\begin{equation}
    Z=BX+\xi,
    \qquad
    \xi\sim\N(0,I_d),
    \qquad
    \xi\perp X.
    \label{eq:generic-gaussian-encoder}
\end{equation}
Then
\begin{align}
    \Cov(X\mid Z)
    &=K,
    \label{eq:posterior-covariance-K}\\
    I(X;Z)
    &=
    -\frac12\log_2\det K,
    \label{eq:rate-from-K}\\
    I(X;Y_t\mid Z)
    &=
    \frac12\log_2\det\!\left(
        I+H_t^{1/2}KH_t^{1/2}
    \right).
    \label{eq:task-distortion-from-K}
\end{align}
\end{lemma}

\begin{proof}
Since $K\preceq I$, the matrix $K^{-1}-I$ is positive semidefinite,
so $B$ is well defined. The standard Gaussian conditioning formula
gives
\begin{align}
    \Cov(X\mid Z) =
    I-B^\top(I+BB^\top)^{-1}B =
    (I+B^\top B)^{-1}.
\end{align}
Because $B^\top B=K^{-1}-I$, $\Cov(X\mid Z)=K$. Next, the mutual information of the linear Gaussian channel is
\begin{align}
    I(X;Z) =\frac12\log_2\det(I+B^\top B) =
    -\frac12\log_2\det K.
\end{align}

Conditioned on $Z$, the covariance of $Y_t$ is
\[
    \Cov(Y_t\mid Z)
    =
    N_t+A_tKA_t^\top.
\]
Moreover, $Y_t$ conditioned on $(X,Z)$ has covariance $N_t$.
Therefore
\begin{align}
    I(X;Y_t\mid Z)
    &=
    h(Y_t\mid Z)-h(Y_t\mid X,Z)\\
    &=
    \frac12\log_2
    \frac{\det(N_t+A_tKA_t^\top)}{\det N_t}\\
    &=
    \frac12\log_2\det\!\left(
        I+N_t^{-1/2}A_tKA_t^\top N_t^{-1/2}
    \right).
\end{align}
Sylvester's determinant identity gives
\[
    \det\!\left(
        I+N_t^{-1/2}A_tKA_t^\top N_t^{-1/2}
    \right)
    =
    \det\!\left(
        I+H_t^{1/2}KH_t^{1/2}
    \right),
\]
which proves \eqref{eq:task-distortion-from-K}.
\end{proof}

\subsection{Geodesic geometry}

For $K_0,K_1\in\mathbb S_{++}^d$ and $s\in[0,1]$, let
\begin{equation}
    K_0\#_sK_1
    =
    K_0^{1/2}
    \left(
        K_0^{-1/2}K_1K_0^{-1/2}
    \right)^s
    K_0^{1/2}
    \label{eq:matrix-geodesic}
\end{equation}
denote their affine-invariant geodesic.

\begin{lemma}[Geodesic convexity]
\label{lem:geodesic-convexity}
The feasible set
\[
    \mathcal K_R
    =
    \{0\prec K\preceq I:\rho(K)\le R\}
\]
is geodesically convex. The function $\rho$ is geodesically affine:
\begin{equation}
    \rho(K_0\#_sK_1)
    =
    (1-s)\rho(K_0)+s\rho(K_1).
    \label{eq:rho-geodesically-affine}
\end{equation}
Moreover, for every $H\succeq0$, the function
\begin{equation}
    f_H(K)=\log\det(I+H^{1/2}KH^{1/2})
\end{equation}
is geodesically convex.
\end{lemma}

\begin{proof}
The determinant of the matrix geometric mean satisfies
\[
    \det(K_0\#_sK_1)
    =
    (\det K_0)^{1-s}(\det K_1)^s,
\]
which proves \eqref{eq:rho-geodesically-affine}. Monotonicity of the
matrix geometric mean gives
\[
    K_0\#_sK_1\preceq I\#_sI=I
\]
whenever $K_0,K_1\preceq I$. Thus $\mathcal K_R$ is geodesically
convex.

To verify convexity of $f_H$, it suffices to compute its second
derivative along a geodesic. At any point on such a geodesic, a
congruence transformation reduces the local calculation to
\[
    K(u)=e^{uL},
    \qquad L=L^\top.
\]
After transforming $H$ accordingly, write
\[
    B\succeq0,
    \qquad
    C=B(I+B)^{-1}.
\]
Using Sylvester's identity,
\[
    f_H(K(u))
    =
    \log\det\!\left(
        I+B^{1/2}e^{uL}B^{1/2}
    \right).
\]
At $u=0$, differentiation gives
\begin{equation}
    \frac{\mathrm d^2}{\mathrm du^2}
    f_H(K(u))\bigg|_{u=0}
    =
    \tr(CL^2)-\tr(CLCL).
    \label{eq:geodesic-hessian}
\end{equation}
Diagonalize $C=\diag(c_1,\ldots,c_d)$, where $0\le c_i<1$.
Then
\begin{align}
    \tr(CL^2)-\tr(CLCL)
    &=
    \sum_i c_i(1-c_i)L_{ii}^2
    \notag\\
    &\quad+
    \sum_{i<j}
    \left[
        c_i(1-c_j)+c_j(1-c_i)
    \right]L_{ij}^2
    \ge0.
    \label{eq:geodesic-hessian-positive}
\end{align}
The same calculation applies at every point of the geodesic, proving
geodesic convexity.
\end{proof}

\begin{lemma}[Strict geodesic convexity]
\label{lem:strict-geodesic-convexity}
If $\overline H=\sum_t\pi_tH_t\succ0$, then
\[
    D_\pi(K)
    =
    \frac12\sum_t\pi_t
    \log_2\det(I+H_t^{1/2}KH_t^{1/2})
\]
is strictly geodesically convex on $\mathbb S_{++}^d$.
\end{lemma}

\begin{proof}
Consider a nonconstant geodesic and perform the same local
normalization as in Lemma~\ref{lem:geodesic-convexity}. Let $L\ne0$
be its local tangent and let
\[
    B_t=K^{1/2}H_tK^{1/2}.
\]
The equality characterization in
\eqref{eq:geodesic-hessian-positive} shows that the second derivative
of the $t$th summand vanishes only if $B_tL=0$. If the second derivative of every positive-weight summand vanished,
then
\[
    \left(\sum_t\pi_tB_t\right)L=0.
\]
But
\[
    \sum_t\pi_tB_t
    =
    K^{1/2}\overline H K^{1/2}\succ0,
\]
so this would imply $L=0$, contrary to the geodesic being
nonconstant. Thus every nonconstant geodesic has strictly positive
second derivative in at least one positive-weight task term.
Therefore $D_\pi$ is strictly geodesically convex.
\end{proof}

\subsection{Existence, uniqueness, and rate exhaustion}

\begin{lemma}[Existence and uniqueness]
\label{lem:existence-uniqueness}
For every finite $R\ge0$, the matrix program
\[
    \min_{K\in\mathcal K_R}D_\pi(K)
\]
has a unique minimizer $K_R$. If $R>0$, then $\rho(K_R)=R$. For $R=0$, one has $K_0=I$.
\end{lemma}

\begin{proof}
If $K\in\mathcal K_R$, then $\det K\ge2^{-2R}$ and every eigenvalue of $K$ is at most one. Consequently,
\[
    \lambda_{\min}(K)\ge2^{-2R}.
\]
Thus $\mathcal K_R$ is contained in the compact set
\[
    \{K:2^{-2R}I\preceq K\preceq I\}.
\]
It is closed within that set, and $D_\pi$ is continuous, so a
minimizer exists.

If $K_0$ and $K_1$ were distinct minimizers, then their geodesic
midpoint would be feasible by
Lemma~\ref{lem:geodesic-convexity}. Strict geodesic convexity would
give
\[
    D_\pi(K_0\#_{1/2}K_1)
    <
    \frac12D_\pi(K_0)+\frac12D_\pi(K_1),
\]
contradicting optimality. The minimizer is therefore unique.

To prove rate exhaustion, define
\begin{equation}
    G(K)
    =
    \sum_t\pi_t
    A_t^\top(N_t+A_tKA_t^\top)^{-1}A_t.
    \label{eq:GK-definition}
\end{equation}
The Euclidean gradient is
\begin{equation}
    \nabla D_\pi(K)
    =
    \frac{G(K)}{2\ln2}.
    \label{eq:D-gradient}
\end{equation}
For every nonzero $x$,
\begin{align}
    x^\top G(K)x
    &=
    \sum_t\pi_t
    (A_tx)^\top
    (N_t+A_tKA_t^\top)^{-1}
    (A_tx).
\end{align}
This vanishes only if $A_tx=0$ for every positive-weight task,
equivalently if $H_tx=0$ for every such task. The assumption
$\overline H\succ0$ rules this out, so $G(K)\succ0$. Hence $D_\pi(K)$ is strictly increasing in the Loewner order.

Suppose $R>0$ and $\rho(K_R)<R$. For $\alpha<1$ sufficiently close
to one,
\[
    \rho(\alpha K_R)
    =
    \rho(K_R)-\frac d2\log_2\alpha
    \le R.
\]
Thus $\alpha K_R$ remains feasible, while $\alpha K_R\prec K_R$ and therefore
\[
    D_\pi(\alpha K_R)<D_\pi(K_R),
\]
which is a contradiction. Hence $\rho(K_R)=R$.

Finally, if $R=0$, feasibility requires
\[
    K\preceq I,
    \qquad
    \det K\ge1.
\]
All eigenvalues must therefore equal one, so $K_0=I$.
\end{proof}

\subsection{Constrained reduction}

We now pass from Gaussian Lagrangian extremality to the
rate-constrained problem.

\begin{lemma}[Supporting multiplier]
\label{lem:supporting-multiplier}
For every $R>0$, there exists $\beta_R>0$ such that the constrained
optimizer $K_R$ also minimizes
\begin{equation}
    K\longmapsto D_\pi(K)+\beta_R\rho(K)
    \label{eq:penalized-Gaussian-objective}
\end{equation}
over all $0\prec K\preceq I$.
\end{lemma}

\begin{proof}
For $R>0$, the matrix program satisfies a strict feasibility
condition. Indeed, choose $\alpha\in(0,1)$ sufficiently close to one
that
\[
    -\frac d2\log_2\alpha<R.
\]
Then $\alpha I\prec I$ and $\rho(\alpha I)<R$.

The first-order optimality conditions therefore provide a scalar
$\beta_R\ge0$ and a matrix $\Lambda_R\succeq0$ such that
\begin{align}
    G(K_R)
    &=
    \beta_RK_R^{-1}-\Lambda_R,
    \label{eq:KKT-stationarity}\\
    \left\langle\Lambda_R,I-K_R\right\rangle
    &=0.
    \label{eq:KKT-complementarity}
\end{align}
Because $G(K_R)\succ0$, equation
\eqref{eq:KKT-stationarity} cannot hold with $\beta_R=0$.
Hence $\beta_R>0$.

The function
\[
    D_\pi(K)+\beta_R\rho(K)
\]
is geodesically convex because $D_\pi$ is geodesically convex and
$\rho$ is geodesically affine. Equations
\eqref{eq:KKT-stationarity}--\eqref{eq:KKT-complementarity} give its
first-order optimality certificate over the geodesically convex set
$0\prec K\preceq I$. Therefore $K_R$ is a global minimizer of
\eqref{eq:penalized-Gaussian-objective}.
\end{proof}

\begin{proof}[Proof of Proposition~\ref{prop:gaussian}]
First suppose $R>0$, and let $K_R$ be the unique optimizer of the
matrix program. By Lemma~\ref{lem:supporting-multiplier}, there is a
$\beta_R>0$ for which
\begin{equation}
    \inf_{0\prec K\preceq I}
    \left\{
        D_\pi(K)+\beta_R\rho(K)
    \right\}
    =
    D_\pi(K_R)+\beta_RR.
    \label{eq:Gaussian-supporting-value}
\end{equation}

Let $P_{Z\mid X}$ be any encoder satisfying $I(X;Z)\le R$.
Gaussian extremality and
\eqref{eq:Gaussian-supporting-value} give
\begin{align}
    D_\pi(Z)+\beta_RI(X;Z)
    &\ge
    D_\pi(K_R)+\beta_RR.
\end{align}
Therefore
\begin{align}
    D_\pi(Z)
    &\ge
    D_\pi(K_R)
    +
    \beta_R\bigl(R-I(X;Z)\bigr) \ge
    D_\pi(K_R).
    \label{eq:arbitrary-encoder-lower-bound}
\end{align}
Thus no arbitrary, possibly non-Gaussian, encoder of rate at most $R$
can improve on the matrix-program value.

Conversely, Lemma~\ref{lem:gaussian-realization} constructs a
Gaussian encoder with residual covariance $K_R$ and
\[
    I(X;Z)=\rho(K_R)=R.
\]
Its distortion equals $D_\pi(K_R)$, so equality is attained in
\eqref{eq:arbitrary-encoder-lower-bound}. This proves
\[
    D_\pi^*(R)
    =
    \min_{K\in\mathcal K_R}D_\pi(K).
\]

If $R=0$, then $I(X;Z)=0$ implies that $Z$ is independent of $X$.
Under the Markov condition $Z-X-(Y_t)_t$, it is consequently
independent of every $(X,Y_t)$ pair, and
\[
    I(X;Y_t\mid Z)=I(X;Y_t).
\]
This agrees with the covariance formula at $K=I$, which is the unique
element of $\mathcal K_0$.

Existence, uniqueness, and rate exhaustion follow from
Lemma~\ref{lem:existence-uniqueness}, while the encoder realization
and its posterior covariance follow from
Lemma~\ref{lem:gaussian-realization}.
\end{proof}

\subsection{Numerical optimality conditions}

For completeness, the unique covariance $K_R$ can be characterized
by explicit first-order conditions.

\begin{proposition}[Optimality conditions]
\label{prop:gaussian-KKT}
Fix $R>0$, and define $G(K)$ by \eqref{eq:GK-definition}. A matrix
$K\in\mathbb S_{++}^d$ is equal to the unique optimizer $K_R$ if and
only if
\begin{align}
    K&\preceq I,
    &
    -\frac12\log_2\det K&=R,
    \label{eq:KKT-primal}\\
    G(K)&=\beta K^{-1}-\Lambda,
    &
    \Lambda&\succeq0,
    \label{eq:KKT-dual}\\
    \left\langle\Lambda,I-K\right\rangle&=0
    \label{eq:KKT-comp}
\end{align}
for some $\beta>0$.

Equivalently, defining
\begin{equation}
    Q(K)=K^{1/2}G(K)K^{1/2},
    \label{eq:QK-definition}
\end{equation}
the conditions are
\begin{align}
    Q(K)&\preceq\beta I,
    \label{eq:KKT-Q-bound}\\
    \bigl(\beta I-Q(K)\bigr)(K^{-1}-I)&=0.
    \label{eq:KKT-Q-complementarity}
\end{align}
\end{proposition}

\begin{proof}
Necessity follows from the multiplier argument in
Lemma~\ref{lem:supporting-multiplier}. Since both
$\Lambda$ and $I-K$ are positive semidefinite,
\[
    \langle\Lambda,I-K\rangle=0
\]
is equivalent to the matrix complementarity relation between their
supports.

For the equivalent form, multiply
\[
    \Lambda=\beta K^{-1}-G(K)
\]
on both sides by $K^{1/2}$ to obtain
\[
    K^{1/2}\Lambda K^{1/2}
    =
    \beta I-Q(K)\succeq0.
\]
Moreover,
\begin{align}
    \langle\Lambda,I-K\rangle
    &=
    \tr\!\left[
        \bigl(\beta I-Q(K)\bigr)(K^{-1}-I)
    \right].
\end{align}
Both matrices in the final trace are positive semidefinite, so the
trace vanishes if and only if their product vanishes. This yields
\eqref{eq:KKT-Q-complementarity}.

Conversely, suppose the stated conditions hold. They give a
first-order optimality certificate for
\[
    D_\pi(K)+\beta\rho(K)
\]
over $0\prec K\preceq I$. This objective is geodesically convex, so
the certificate is globally sufficient. Any feasible $K'$ with
$\rho(K')\le R=\rho(K)$ then satisfies
\begin{align}
    D_\pi(K')+\beta\rho(K')
    &\ge
    D_\pi(K)+\beta\rho(K),
\end{align}
and hence
\[
    D_\pi(K')
    \ge
    D_\pi(K)+\beta\bigl(R-\rho(K')\bigr)
    \ge
    D_\pi(K).
\]
Thus $K$ solves the constrained program. Uniqueness implies
$K=K_R$.
\end{proof}

The conditions have the usual water-filling interpretation. On every
encoded direction, meaning a direction in the range of $I-K_R$, the
normalized marginal matrix
\[
    Q(K_R)=K_R^{1/2}G(K_R)K_R^{1/2}
\]
equals the common multiplier $\beta_R$. On unencoded directions it is
bounded above by $\beta_R$.

\clearpage
\section{Task Knowledge: Crossover Proofs}
\label{app:knowledgeproofs}

\subsection{Optimal Specialist and Generalist Covariances}
We first justify \eqref{eq:focuscoverage}. Work in an orthonormal basis
whose first $k$ coordinates span $S$. Partition any feasible covariance
as
\[
K=\begin{pmatrix}A&B\\B^\top&C\end{pmatrix}.
\]
The specialist objective depends only on $A$. Replacing $K$ by
$\operatorname{diag}(A,I_{d-k})$ preserves this objective, is feasible,
and weakly reduces the required rate: the Schur complement satisfies
\[
0\prec C-B^\top A^{-1}B\preceq I,\qquad
\det K=\det A\,\det(C-B^\top A^{-1}B)\le\det A.
\]
We may therefore optimize on $S$ and leave $S^\perp$ unencoded.
On $S$, the specialist family has positive-definite average information
matrix $(h/k)I_k$. Its unique optimum, by the Gaussian reduction, is
invariant under all coordinate sign changes and permutations, since both
the objective and constraints have these symmetries.
It must consequently be a scalar multiple of $I_k$.
The objective strictly decreases when this scalar decreases, so all
available rate is used and $A=2^{-2R/k}I_k$.
Equivalently, its diagonal allocation problem is
\[
\min_{\substack{r_i\ge0\\\sum_{i=1}^k r_i\le R}}
\frac1{2k}\sum_{i=1}^k\log_2(1+h\,2^{-2r_i}),
\]
whose strictly convex, decreasing summands yield $r_i=R/k$.

For the generalist, $H_g=\sigma_g^{-2}I_d$. Orthogonal invariance
and uniqueness force its optimizer to be scalar, and rate saturation
gives $K_g(R)=2^{-2R/d}I_d$.
For squared-error reconstruction, the Gaussian distortion is $\tr K$.
The arithmetic--geometric mean inequality gives, for every
$K\in\mathcal K_R$,
\[
\tr K\ge d(\det K)^{1/d}\ge d\,2^{-2R/d},
\]
with equality at the same isotropic covariance.
This reconstruction statement concerns mean-square error; a noiseless
continuous target in the IB relevance objective gives
$I(Z;V_g)=I(Z;X)$ and does not distinguish rate-saturating geometries.

\subsection{Proof of Proposition~\ref{prop:crossovers}}
The differential of the evaluation distortion at $K=I$ is
\[
\mathrm dD_\pi(I)[E]
=\frac1{2\ln2}
\sum_t\pi_t\tr\!\left(
H_t^{1/2}(I+H_t)^{-1}H_t^{1/2}E
\right)
=\frac{\tr(\overline C E)}{2\ln2}.
\]
The covariances in \eqref{eq:focuscoverage} satisfy
\[
K_s(R)=I-\frac{2\ln2}{k}RP+O(R^2),\qquad
K_g(R)=I-\frac{2\ln2}{d}RI+O(R^2).
\]
Since $D_\pi$ is smooth near $I$, substitution yields
\[
D_\pi(K_g(R))-D_\pi(K_s(R))
=R\left\{\frac{\tr(P\overline C)}{k}
-\frac{\tr(\overline C)}{d}\right\}+O(R^2).
\]
The strict positivity of $\Delta$ proves the low-rate inequality.

The log-determinant expression for $D_\pi$ extends continuously to
positive semidefinite $K$. As $R\to\infty$,
$K_s(R)\to I-P$ and $K_g(R)\to0$, giving respectively
$D_{\rm miss}>0$ and zero. This proves the high-rate inequality.
Continuity and the opposite strict signs at finite small and large
rates imply at least one positive finite crossing; uniqueness of that
crossing is not asserted.

To characterize the missing-information condition, each term in
\[
D_{\rm miss}
=\frac12\sum_t\pi_t
\log_2\det\!\left(I+H_t^{1/2}(I-P)H_t^{1/2}\right)
\]
is nonnegative and is positive exactly when its positive semidefinite
matrix $H_t^{1/2}(I-P)H_t^{1/2}$ is nonzero.
For positive task weights this is equivalent to
$\sum_t\pi_t\tr((I-P)H_t)>0$, or
$\tr((I-P)\overline H)>0$.

An explicit sufficient high-rate threshold follows from monotonicity:
$K_s(R)\succeq I-P$ implies $D_\pi(K_s(R))\ge D_{\rm miss}$.
Using $\ln\det(I+A)\le\tr A$ for $A\succeq0$,
\[
D_\pi(K_g(R))
\le\frac{\tr\overline H}{2\ln2}\,2^{-2R/d}.
\]
Thus the generalist strictly outperforms the specialist whenever
\[
R>\frac d2\log_2\!\left(
\frac{\tr\overline H}{2(\ln2)D_{\rm miss}}\right).
\]

\clearpage
\section{Fixed-Pretext Regret Bound for the Two-Mode Example}
\label{app:gap}
This argument uses the rank-one, coordinate-specific tasks in
\eqref{eq:diagonalexample}; it is not a pinching argument for general
full-rank diagonal task matrices.
For any two-dimensional fixed pretext, write
$L_R=U\diag(\ell_1(R),\ell_2(R))U^\top$ with
$\ell_1(R)\le\ell_2(R)$. Then
\[
(L_R)_{11}-(L_R)_{22}
=(u_{11}^2-u_{21}^2)(\ell_1(R)-\ell_2(R))
\]
cannot change sign across rates. Thus it has the wrong, or tied,
coordinate ordering at at least one of $R_1,R_2$.

For these rank-one tasks, downstream distortion depends only on the
two diagonal entries. Deleting the off-diagonal entry preserves that
distortion and increases the determinant, weakly reducing the required
rate. The resulting diagonal covariance is feasible and has the same
wrong ordering. If it uses less than $R$ bits, scaling both entries down
to use $R$ bits preserves their ordering and reduces distortion.
It is therefore enough to minimize $F_R(r)$ on the wrong side of
$r=R/2$. Strict convexity places this constrained minimum at $R/2$.
This proves \eqref{eq:examplerelative}, including oblique pretexts.

For completeness, put $L_M=\log_2 M$. At $R_1=L_M$,
\[
D_1^*=\frac p2+\frac q2\log_2(1+M^4),\qquad
D_1^{\rm eq}=\frac p2\log_2(1+M)+\frac q2\log_2(1+M^3).
\]
At $R_2=4L_M$, let $w$ be the positive root of
\[
(M-1)w^2+(M-2)w-M^2=0.
\]
The optimum has
\[
r_1^*=2L_M-\tfrac12\log_2w,\qquad
r_2^*=2L_M+\tfrac12\log_2w,
\]
and
\[
D_2^*=\frac p2\log_2(1+w/M^2)+\frac q2\log_2(1+1/w),\qquad
D_2^{\rm eq}=\frac p2\log_2(1+M^{-2})+\frac q2.
\]
These expressions give Figure ~\ref{fig:reversal}.
As $M\to\infty$,
\[
D_1^*\to\tfrac12,\quad
\frac{D_1^{\rm eq}-D_1^*}{D_1^*}\sim\log_2M,\quad
D_2^*\sim\frac{M^{-3/2}}{\ln2},\quad
\frac{D_2^{\rm eq}-D_2^*}{D_2^*}\sim\frac{\ln2}{2}\sqrt M.
\]
Thus the worst-rate relative regret of every fixed Gaussian pretext can
be arbitrarily large, despite commuting tasks and perfectly nested
supports. This family uses increasing signal strengths and unequal task
weights; it is an existence result, not a claim about typical data.

\clearpage
\section{Proofs for Rate-Matched Gaussian Pretexts}
\label{app:matched}

Throughout this appendix, let
\[
    \rho(K)=-\frac12\log_2\det K
\]
and let $\mathcal K_R$ be defined by
\eqref{eq:gaussian-feasible-set}. We assume that $N_t\succ0$, that
$\pi_t>0$ with $\sum_t\pi_t=1$, and that
\[
    \overline H=\sum_t\pi_tA_t^\top N_t^{-1}A_t\succ0.
\]

\subsection{Preliminary identities}

For each task, define
\begin{equation}
    B_t(K)
    =
    A_t^\top(N_t+A_tKA_t^\top)^{-1}A_t.
\end{equation}
The differential of the downstream objective is
\begin{equation}
    \mathrm dD_\pi(K)[\Delta]
    =
    \frac{1}{2\ln2}
    \sum_t\pi_t\tr\!\left(B_t(K)\Delta\right),
\end{equation}
and hence
\begin{equation}
    \nabla D_\pi(K)
    =
    \frac{1}{2\ln2}\sum_t\pi_tB_t(K).
    \label{eq:down-gradient}
\end{equation}

Similarly, for a pretext with corruption covariance $\Sigma\succ0$,
\begin{align}
    D_\Sigma(K)
    &=
    \frac{1}{2\ln2}
    \left[
        \ln\det(\Sigma+K)-\ln\det\Sigma
    \right],\\
    \nabla D_\Sigma(K)
    &=
    \frac{(\Sigma+K)^{-1}}{2\ln2}.
    \label{eq:pretext-gradient}
\end{align}

Letting $H_t=A_t^\top N_t^{-1}A_t$, the Woodbury identity gives
\begin{align}
    B_t(K)
    &=
    H_t-H_t(K^{-1}+H_t)^{-1}H_t
    \label{eq:B-woodbury-1}\\
    &=
    K^{-1}
    -
    K^{-1}(K^{-1}+H_t)^{-1}K^{-1}.
    \label{eq:B-woodbury-2}
\end{align}
In particular,
\begin{equation}
    0\preceq B_t(K)\prec K^{-1}.
    \label{eq:B-bounds}
\end{equation}
Moreover, $\ker B_t(K)=\ker A_t=\ker H_t$. It follows that
\begin{equation}
    \sum_t\pi_tB_t(K)\succ0
\end{equation}
whenever $\overline H\succ0$.

\subsection{Proof of Theorem~\ref{thm:matched}}

\begin{proof}
Fix $R>0$ and a downstream-optimal covariance
$K_R\in\mathcal K_R$. Define
\[
    G_R=\sum_t\pi_tB_t(K_R).
\]
Because $\overline H\succ0$, the common null space of the positive-weight
task matrices is trivial. Therefore $G_R\succ0$. By \eqref{eq:B-bounds} and $\sum_t\pi_t=1$,
\[
    G_R
    \prec
    \sum_t\pi_tK_R^{-1}
    =
    K_R^{-1}.
\]
Inversion reverses the strict Loewner order, so $G_R^{-1}\succ K_R$, and consequently,
\[
    \Sigma_a^*(R)=G_R^{-1}-K_R\succ0.
\]

It remains to prove that $K_R$ is globally optimal for the matched
pretext. By construction,
\[
    \Sigma_a^*(R)+K_R=G_R^{-1}.
\]
Using \eqref{eq:down-gradient} and
\eqref{eq:pretext-gradient},
\begin{align}
    \nabla D_{\Sigma_a^*(R)}(K_R) =
    \frac{(\Sigma_a^*(R)+K_R)^{-1}}{2\ln2} =
    \frac{G_R}{2\ln2} =
    \nabla D_\pi(K_R).
    \label{eq:matched-gradients}
\end{align}

For completeness, consider any $K\in\mathcal K_R$ and the
affine-invariant geodesic
\[
    \gamma(s)=K_R\#_sK,
    \qquad 0\le s\le1.
\]
The set $\mathcal K_R$ is geodesically convex. Indeed,
\[
    \det(K_R\#_sK)
    =
    (\det K_R)^{1-s}(\det K)^s,
\]
so $\rho(\gamma(s))$ is affine in $s$, and monotonicity of the matrix
geometric mean gives $\gamma(s)\preceq I$ whenever
$K_R,K\preceq I$.

Because $K_R$ minimizes $D_\pi$ on $\mathcal K_R$,
\[
    \left.\frac{\mathrm d}{\mathrm ds}
    D_\pi(\gamma(s))\right|_{s=0^+}
    \ge0.
\]
Equation~\eqref{eq:matched-gradients} implies
\[
    \left.\frac{\mathrm d}{\mathrm ds}
    D_{\Sigma_a^*(R)}(\gamma(s))\right|_{s=0^+}
    =
    \left.\frac{\mathrm d}{\mathrm ds}
    D_\pi(\gamma(s))\right|_{s=0^+}
    \ge0.
\]
The map
\[
    K\mapsto\log\det(\Sigma+K)
\]
is geodesically convex for every $\Sigma\succ0$. Therefore
$s\mapsto D_{\Sigma_a^*(R)}(\gamma(s))$ is convex, and
\begin{align}
    D_{\Sigma_a^*(R)}(K)
    &\ge
    D_{\Sigma_a^*(R)}(K_R)
    +
    \left.\frac{\mathrm d}{\mathrm ds}
    D_{\Sigma_a^*(R)}(\gamma(s))\right|_{s=0^+}\\
    &\ge
    D_{\Sigma_a^*(R)}(K_R).
\end{align}
Since this holds for every $K\in\mathcal K_R$, the covariance $K_R$ is
a global optimum of the matched-pretext problem. Evaluating the
downstream objective at this same covariance gives
\[
    \Reg_\pi(a_R,R)=0.
\]
\end{proof}

\paragraph{Nonuniqueness.}
For any $c>0$, let
\[
    H_a=cK_R^{-1},
    \qquad
    \Sigma_a=c^{-1}K_R.
\]
Then
\[
    \nabla D_a(K_R)
    =
    \frac{c}{c+1}\frac{K_R^{-1}}{2\ln2},
\]
which is proportional to the gradient of the active rate constraint.
The same geodesic first-order argument shows that $K_R$ is also an
optimum for this pretext. Thus Theorem~\ref{thm:matched} identifies a
task-dependent matched pretext, not the unique pretext inducing $K_R$.

\subsection{Endpoint behavior of the downstream optimum}
\label{app:endpoint-expansions}

We first note that the rate constraint is exhausted: $\rho(K_R)=R$. Indeed, if $\rho(K_R)<R$, then $(1-\epsilon)K_R$ is feasible for
sufficiently small $\epsilon>0$. Since $\overline H\succ0$,
$D_\pi((1-\epsilon)K_R)<D_\pi(K_R)$, contradicting optimality.
Consequently,
\begin{equation}
    \det K_R=2^{-2R}.
    \label{eq:det-KR}
\end{equation}

\subsubsection{Low-rate limit}

Let
\[
    E_R=I-K_R\succeq0.
\]
Equation~\eqref{eq:det-KR} implies $K_R\to I$, and hence $E_R\to0$, as
$R\downarrow0$. Expanding the rate gives
\begin{align}
    2(\ln2)R =
    -\ln\det(I-E_R) =
    \tr E_R+O(\|E_R\|_F^2).
    \label{eq:rate-low-proof}
\end{align}
In particular,
\begin{equation}
    \frac{\tr E_R}{2(\ln2)R}\longrightarrow1.
    \label{eq:trace-E}
\end{equation}

At $K=I$,
\begin{align}
    B_t(I) = A_t^\top(N_t+A_tA_t^\top)^{-1}A_t =
    H_t(I+H_t)^{-1} = C_t.
\end{align}
A first-order Taylor expansion therefore gives
\begin{equation}
    D_\pi(I-E)
    =
    D_\pi(I)
    -
    \frac{\tr(\overline C E)}{2\ln2}
    +O(\|E\|_F^2).
    \label{eq:D-low-proof}
\end{equation}

Let $\lambda_1$ be the largest eigenvalue of $\overline C$, with unit
eigenvector $v$. The rank-one covariance
\[
    \widehat K_R
    =
    I-\left(1-2^{-2R}\right)vv^\top
\]
is feasible and satisfies $\det\widehat K_R=2^{-2R}$. Comparing $K_R$
with $\widehat K_R$ in \eqref{eq:D-low-proof} shows that
\begin{equation}
    \frac{\tr(\overline C E_R)}{\tr E_R}
    \longrightarrow\lambda_1.
    \label{eq:Rayleigh-limit}
\end{equation}

The matrices
\[
    Q_R=\frac{E_R}{\tr E_R}
\]
are positive semidefinite with unit trace. Every accumulation point
$Q$ therefore satisfies
\[
    Q\succeq0,\qquad \tr Q=1,\qquad
    \tr(\overline C Q)=\lambda_1.
\]
If $\lambda_1$ is simple, equality in the maximal-eigenvalue bound
forces $Q=vv^\top$. Thus
\[
    \frac{E_R}{\tr E_R}\longrightarrow vv^\top.
\]
Combining this with \eqref{eq:trace-E} yields
\[
    \frac{I-K_R}{2(\ln2)R}
    \longrightarrow vv^\top.
\]

\subsubsection{High-rate limit}

For $K\to0$, Taylor expansion of the logarithmic determinant gives
\begin{equation}
    D_\pi(K)
    =
    \frac{\tr(\overline H K)}{2\ln2}
    +O(\|K\|_F^2).
    \label{eq:D-high-proof}
\end{equation}
Let
\[
    \delta_R=2^{-2R},
    \qquad
    c_R=(\det\overline H)^{1/d}\delta_R^{1/d}.
\]
The candidate $\widehat K_R=c_R\overline H^{-1}$ satisfies
\[
    \det\widehat K_R
    =
    \frac{c_R^d}{\det\overline H}
    =
    \delta_R.
\]
It is therefore feasible for all sufficiently large $R$. Moreover,
\[
    D_\pi(\widehat K_R)
    =
    \frac{dc_R}{2\ln2}+O(c_R^2).
\]
Optimality implies $D_\pi(K_R)=O(c_R)$.

There exists a constant $\alpha>0$, independent of
$0\prec K\preceq I$, such that
\[
    D_\pi(K)\ge\alpha\,\tr(\overline H K).
\]
For example, this follows from
$\ln(1+x)\ge x/(1+M)$ on $[0,M]$, with
$M=\max_t\|H_t\|_{\mathrm{op}}$. Hence
\[
    \tr(\overline H K_R)=O(c_R).
\]
Since $\overline H\succ0$, this implies $\|K_R\|_{\mathrm{op}}=O(c_R)$.

Define
\[
    L_R=\frac{K_R}{c_R}.
\]
The sequence $(L_R)$ is bounded and
\begin{equation}
    \det L_R
    =
    \frac{\delta_R}{c_R^d}
    =
    \frac{1}{\det\overline H}.
    \label{eq:det-L}
\end{equation}
Dividing \eqref{eq:D-high-proof} by $c_R$ and comparing with
$\widehat K_R$ shows that every accumulation point $L$ of $(L_R)$
minimizes $\tr(\overline H L)$ subject to
\[
    L\succ0,
    \qquad
    \det L=(\det\overline H)^{-1}.
\]
By the AM-GM inequality,
\begin{align}
    \tr(\overline H L) =
    \tr\!\left(\overline H^{1/2}L\overline H^{1/2}\right) \ge
    d\det\!\left(
        \overline H^{1/2}L\overline H^{1/2}
    \right)^{1/d} =d.
\end{align}
Equality holds uniquely when $\overline H^{1/2}L\overline H^{1/2}=I$, or equivalently when $L=\overline H^{-1}$. Every accumulation point is therefore $\overline H^{-1}$, proving
\[
    \frac{K_R}{c_R}\longrightarrow\overline H^{-1}.
\]

\subsection{Proof of Corollary~\ref{cor:limits}}

\begin{proof}
Recall that
\[
    G(K)
    =
    \sum_t\pi_t
    A_t^\top(N_t+A_tKA_t^\top)^{-1}A_t
\]
is continuous on $K\succeq0$.

As $R\downarrow0$, the rate constraint gives $K_R\to I$. Hence
\begin{align}
    G_R
    &\longrightarrow
    G(I)\\
    &=
    \sum_t\pi_t
    A_t^\top(N_t+A_tA_t^\top)^{-1}A_t =
    \sum_t\pi_tH_t(I+H_t)^{-1} = \overline C.
\end{align}
Since $\overline C\succ0$, inversion is continuous at $\overline C$.
Therefore
\[
    \Sigma_a^*(R)
    =
    G_R^{-1}-K_R
    \longrightarrow
    \overline C^{-1}-I.
\]

As $R\to\infty$, equation~\eqref{eq:high} gives $K_R\to0$. Consequently,
\begin{align}
    G_R \longrightarrow
    G(0) =
    \sum_t\pi_tA_t^\top N_t^{-1}A_t = \overline H.
\end{align}
Thus
\[
    \Sigma_a^*(R)
    =
    G_R^{-1}-K_R
    \longrightarrow
    \overline H^{-1},
\]
completing the proof.
\end{proof}

\clearpage
\section{Proof of the Bounded-SNR Approximation}
\label{app:bounded-snr}
All logarithms in this section of the appendix are natural unless explicitly marked otherwise. Define
\begin{equation}
F(K)=2\ln2\,D_\pi(K)
=\sum_t\pi_t\ln\det(I+H_t^{1/2}KH_t^{1/2}).
\label{eq:snr-F}
\end{equation}
We first prove a comparison for linear covariance objectives. The argument
does not assume simultaneous diagonalization.

\begin{lemma}[Trace comparison under bounded matrix mismatch]
\label{lem:snr-trace}
Let $A,G\succ0$ satisfy $G\preceq A\preceq\kappa G$ for some
$\kappa\ge1$. Fix a finite $R\ge0$ and let
\[
\widehat K\in\arg\min_{K\in\mathcal K_R}\operatorname{tr}(AK).
\]
Then, for every $K\in\mathcal K_R$,
\begin{equation}
\operatorname{tr}(G\widehat K)
\le S(\kappa)\operatorname{tr}(GK).
\label{eq:snr-trace-comparison}
\end{equation}
\end{lemma}

\begin{proof}
If $\kappa=1$, then $A=G$ and the claim follows from optimality.
Suppose $\kappa>1$. Given any feasible $K$, write
\[
V=\log(\widehat K^{-1/2}K\widehat K^{-1/2}),\qquad
K(s)=\widehat K^{1/2}e^{sV}\widehat K^{1/2},\quad 0\le s\le1.
\]
This path remains feasible. Indeed,
\[
\ln\det K(s)=(1-s)\ln\det\widehat K+s\ln\det K,
\]
so the rate constraint is preserved. Applying
$x^s\le(1-s)+sx$ to the eigenvalues of $e^V$ gives
\[
K(s)\preceq(1-s)\widehat K+sK\preceq I.
\]
Optimality at $s=0$ therefore implies
\begin{equation}
\operatorname{tr}(CV)\ge0,\qquad
C=\widehat K^{1/2}A\widehat K^{1/2}.
\label{eq:snr-directional}
\end{equation}
Let $B=\widehat K^{1/2}G\widehat K^{1/2}$, so
$B\preceq C\preceq\kappa B$. In an orthonormal eigenbasis of $V$,
write its eigenvalues as $v_i$ and the diagonal entries of $B,C$ as
$b_i,c_i$. Then
\[
0<b_i\le c_i\le\kappa b_i,\qquad \sum_i c_iv_i\ge0.
\]
Define probability vectors
$P_i=c_i/\sum_jc_j$ and $Q_i=b_i/\sum_jb_j$.
By Jensen's inequality and \eqref{eq:snr-directional},
\begin{align}
\frac{\operatorname{tr}(GK)}{\operatorname{tr}(G\widehat K)}
&=\frac{\operatorname{tr}(Be^V)}{\operatorname{tr}(B)}
=\sum_iQ_ie^{v_i}\nonumber\\
&=\sum_iP_i\exp\!\left(v_i+\ln\frac{Q_i}{P_i}\right)
\ge \exp\!\left(\sum_iP_iv_i-D_{\rm KL}(P\|Q)\right)
\ge e^{-D_{\rm KL}(P\|Q)}.
\label{eq:snr-jensen}
\end{align}
Here $D_{\rm KL}$ uses natural logarithms. Set
$x_i=c_i/b_i\in[1,\kappa]$ and $t=\sum_iQ_ix_i$. Then
\[
D_{\rm KL}(P\|Q)
=\frac{\sum_iQ_ix_i\ln x_i}{t}-\ln t.
\]
Convexity gives the chord bound
\[
x\ln x\le\frac{\kappa\ln\kappa}{\kappa-1}(x-1),
\qquad 1\le x\le\kappa.
\]
Writing $a=\kappa\ln\kappa/(\kappa-1)$, we obtain
\begin{equation}
D_{\rm KL}(P\|Q)
\le a(1-t^{-1})-\ln t
\le a-1-\ln a
=\ln S(\kappa).
\label{eq:snr-kl}
\end{equation}
The middle expression is maximized over $t>0$ at $t=a$.
Combining \eqref{eq:snr-jensen} and \eqref{eq:snr-kl} proves the lemma.
\end{proof}

\begin{proof}[Proof of Theorem~\ref{thm:bounded-snr}]
Set $A=\overline H$ and fix $R\ge0$. Abbreviate the downstream optimum
as $K_*=K_R$ and the fixed-pretext optimum as $\widehat K=K_R^A$.
The case $R=0$ is immediate because $\mathcal K_0=\{I\}$, so assume
$R>0$.

We first identify $\widehat K$ as a minimizer of the linear objective
$\operatorname{tr}(AK)$. If $A=U\operatorname{diag}(a_i)U^\top$, the
single-task Gaussian allocation rule \eqref{eq:pretextalloc} gives
\begin{equation}
\widehat K=U\operatorname{diag}(\widehat k_i)U^\top,
\qquad
\widehat k_i=\min\{1,\nu/a_i\},
\label{eq:snr-waterfill}
\end{equation}
where $\nu>0$ is selected to exhaust the rate budget. This matrix also
solves the convex program minimizing $\operatorname{tr}(AK)$ over
$\mathcal K_R$. To verify this directly, its KKT stationarity condition is
\[
A-\nu\widehat K^{-1}+M=0,
\qquad
M=U\operatorname{diag}((\nu-a_i)_+)U^\top\succeq0,
\]
with $M(I-\widehat K)=0$ and active constraint
$-\ln\det\widehat K=2R\ln2$. These conditions are sufficient by convexity.

Next let $G=\nabla F(K_*)$, namely
\begin{equation}
G=\sum_t\pi_tH_t^{1/2}
(I+H_t^{1/2}K_*H_t^{1/2})^{-1}H_t^{1/2}.
\label{eq:snr-gradient}
\end{equation}
Since $K_*\preceq I$ and $H_t\preceq LI$, each residual-SNR matrix
$B_t=H_t^{1/2}K_*H_t^{1/2}$ satisfies $0\preceq B_t\preceq LI$.
It follows that
\begin{equation}
\frac{A}{1+L}\preceq G\preceq A,
\qquad\text{or equivalently}\qquad
G\preceq A\preceq(1+L)G.
\label{eq:snr-gradient-sandwich}
\end{equation}
In particular, $G\succ0$. Lemma~\ref{lem:snr-trace} now gives
\[
\operatorname{tr}(G\widehat K)
\le S(1+L)\operatorname{tr}(GK_*).
\]
The function $F$ is concave in the ordinary matrix coordinates, hence
\begin{align}
F(\widehat K)
&\le F(K_*)+\operatorname{tr}(G(\widehat K-K_*))\nonumber\\
&\le F(K_*)+[S(1+L)-1]\operatorname{tr}(GK_*).
\label{eq:snr-final-concavity}
\end{align}
If $\lambda_{ti}$ denote the eigenvalues of $B_t$, then
\[
\operatorname{tr}(GK_*)
=\sum_{t,i}\pi_t\frac{\lambda_{ti}}{1+\lambda_{ti}}
\le\sum_{t,i}\pi_t\ln(1+\lambda_{ti})
=F(K_*).
\]
Substitution in \eqref{eq:snr-final-concavity} and division by $2\ln2$
prove \eqref{eq:bounded-snr-factor}. The same matrix $A$ works at every
rate, proving \eqref{eq:bounded-snr-regret}.

For the expansion, let $x=\ln(1+L)$ and
$a=x/(1-e^{-x})=1+x/2+O(x^2)$. Since
$\ln S(e^x)=a-1-\ln a=x^2/8+O(x^3)$ and
$x=L+O(L^2)$, equation~\eqref{eq:bounded-snr-small-L} follows.
\end{proof}

\paragraph{Task-relevant subspace.}
If $\overline H$ is singular, apply the result after restricting the input
to $\operatorname{range}(\overline H)$, as in Section~\ref{sec:gaussian-task-model}.
Equivalently, allow a positive-semidefinite pretext that ignores the common
task-nullspace. Requiring a strictly positive-definite pretext on additional
task-irrelevant dimensions is a different restriction and need not admit
this uniform all-rate guarantee.

\clearpage
\section{Experiment Details}
\label{app:experiment-details}

This appendix describes the data partitions, training procedures, estimators,
and numerical optimization used in the experiments. Section~\ref{app:celeba-training-details} gives the details for the CelebA experiment, Section~\ref{app:dsprites-training-details} gives the details for the dSprites experiment, and Section~\ref{app:gaussian-solver} describes the algorithm used for the Gaussian covariance optimization \eqref{eq:program}. The main crossover and recovery figures remain in the main text.

\paragraph{Scope and model selection.}
We describe the frozen-backbone CelebA experiment and the main dSprites
experiment used in the paper figures. Within each experiment, pretexts share
the representation architecture, optimization settings, data-partition
protocol, and rate or regularization grid. Test losses are used for reporting,
not for the checkpoint or hyperparameter selection procedures described
below. Architectural dimensions, learning rates, epoch budgets, and numerical
regularization constants are fixed configuration settings. The recorded
experiments do not contain a broader search establishing their optimality.
We distinguish these fixed settings from the explicitly implemented
validation-based choices.

\paragraph{Computational resources.}
The experiments were conducted on computing infrastructure equipped with
four NVIDIA H100 GPUs.

\paragraph{Consent and sensitive content.} CelebA contains identifiable celebrity face images and appearance-related attribute annotations. Its official dataset page states that the images were obtained from the Internet, but the page does not document whether the depicted individuals consented to their inclusion or subsequent research use. Our study reuses the existing dataset and does not collect new data from participants. The supplement redistributes no CelebA images, annotations, or per-image feature caches. dSprites consists of procedurally generated shapes and contains no human-subject imagery.

\subsection{CelebA experiment}
\label{app:celeba-training-details}

\paragraph{Data and frozen features.}
We sample without replacement 10,000 training images, 2,000 validation images,
and 2,000 test images from the official CelebA partitions of sizes
162,770, 19,867, and 19,962, respectively. Subsampling uses NumPy generators
with seeds 42, 43, and 44 for the three partitions. The validation subset
is randomly divided, using seed $42+9182$, into 1,000 images for encoder
checkpoint selection and 1,000 images for selection among pretexts. The
training and test subsets remain unchanged across the five representation
seeds $\{42,43,44,45,46\}$. These seeds therefore do not constitute five
independent data splits.

An ImageNet-1K V1 pretrained ResNet-18, with its final classifier removed,
provides 512-dimensional features. The backbone is frozen and evaluated in
inference mode. Anchor images use deterministic bilinear resizing to a shorter edge of 256,
a $224\times224$ center crop, and channel normalization with means
$(0.485,0.456,0.406)$ and standard deviations $(0.229,0.224,0.225)$. PCA and whitening are fitted exclusively to training
anchor features and retain at most 32 eigen-directions with eigenvalues
exceeding $\max(10^{-5}\lambda_{\max},10^{-12})$. The fitted transform is
applied unchanged to validation and test images. Target-view features use
separately fitted training-only PCA/whitening transforms, also of dimension
at most 32. Feature extraction uses batches of 64 and two data-loader workers.

\paragraph{Tasks and pretexts.}
The four equally weighted evaluation attributes are \texttt{Smiling},
\texttt{Eyeglasses}, \texttt{Blond\_Hair}, and \texttt{Young}. The specialist
shown in the figures predicts five disjoint source attributes:
\texttt{Mouth\_Slightly\_Open}, \texttt{High\_Cheekbones},
\texttt{Wavy\_Hair}, \texttt{Straight\_Hair}, and
\texttt{Receding\_Hairline}. The generalist reconstructs the anchor feature
vector, rather than image pixels. The complete eight-pretext bank additionally
contains lower-region prediction, grayscale-view prediction, separate
expression and hair attribute objectives, an equally weighted
semantic/reconstruction objective, and equally weighted upper/lower-region
prediction. The task-aware baseline uses the four evaluation attributes
on the training partition. It is distinct from the source-label pretexts. 

The average-task baseline uses a single-target Gaussian pretext with information matrix \(\overline H=\sum_t\pi_tH_t\), formed from the fitted evaluation-task matrices, as in Theorem \ref{thm:matched}. Unlike the task-aware baseline, which optimizes the weighted mixture of individual task losses, this baseline optimizes the single Gaussian prediction objective defined by their average information matrix.

Two reproducible target views per configured image policy are cached.
Target augmentations use a random resized crop to $224\times224$ with crop
area fraction in $[0.9,1]$ and horizontal-flip probability $0.5$.
Grayscale and region-prediction views use brightness, contrast, and saturation
jitter of $0.1$ and hue jitter of $0.025$. Upper and lower crops retain
horizontal fractions $[0.08,0.92]$ and vertical fractions $[0.05,0.62]$ and
$[0.40,0.97]$, respectively, with a vertical shift sampled from
$[-0.025,0.025]$. Views use ImageNet channel normalization. The cached views
are reused during bottleneck training, rather than generated afresh each epoch.

\paragraph{Gaussian surrogates and channels.}
For each scalar attribute, we fit a ridge regression with an intercept and
ridge coefficient $10^{-3}$ on whitened training features. Five-fold shuffled
cross-fitting, with seed 42, estimates the residual mean squared error.
The final mean predictor is refitted on all training examples. With fitted
slope $b_t$ and residual variance
$\hat\sigma_t^2=\max(\mathrm{MSE}_{\mathrm{OOF},t},10^{-4})$, the task matrix is
$H_t=b_tb_t^\top/\hat\sigma_t^2$. The cross-fitting folds share the unlabeled
PCA/whitening transform fitted on the complete training partition. Preprocessing
is not refitted separately within each regression fold. Feature-prediction
surrogates are obtained from training cross-covariances and canonical
correlations, clipped at $0.99$ before conversion to Gaussian signal-to-noise
ratios. Source-attribute surrogates use the same ridge/cross-fitting procedure.

The nominal Gaussian rate grid is
$\{0,0.25,1,2,4,8,16,32\}$ bits. Given residual covariance $K$, the channel
is $C=AX+\varepsilon$, with independent standard Gaussian noise and
$A^\top A=K^{-1}-I$. Pure-noise output directions are removed. Zero rate
is represented by a constant feature. Two independent channel-noise draws
per image are used for probe fitting and scoring. Random-number seeds are
shared across candidates and budgets. The nominal rate is mutual information
for the fitted whitened Gaussian model, not an empirical measurement of
mutual information for the non-Gaussian image features.

\paragraph{Finite-code training.}
Code budgets are $B\in\{0,2,4,8,16,32\}$ bits per image. For $B>0$, the
encoder is a $32\!\rightarrow\!128\!\rightarrow\!B$ multilayer perceptron
(with the actual retained PCA dimension replacing 32 if necessary), using
GELU after its hidden layer. Forward passes produce hard bits by thresholding
logits at zero. During training, a straight-through estimator uses the
sigmoid derivative while retaining hard bits in the forward pass. Evaluation
uses deterministic hard bits. Each prediction head has architecture
$B\!\rightarrow\!128\!\rightarrow\!m$ with GELU, where $m$ is the target
dimension. Attribute heads minimize mean binary cross-entropy. Feature heads
minimize Gaussian negative log likelihood per target coordinate, with a
learned full residual covariance parameterized by a Cholesky factor whose
diagonal is $\operatorname{softplus}(a)+10^{-3}$. Composite objectives use
separate heads and a shared encoder. At zero bits, no encoder is trained.

Each nonzero-budget encoder is trained independently for 30 epochs using
AdamW, learning rate $10^{-3}$, weight decay $10^{-4}$, batch size 256, and
gradient-norm clipping at 10. Training shuffles images every epoch and samples
one cached target draw per image and update. The checkpoint minimizing the
corresponding hard-code validation objective, averaged over all cached
validation target draws, is retained. This is best-checkpoint selection over
a fixed 30-epoch run, not early termination. The trained task-aware code
baseline uses the same procedure and is not claimed to be globally optimal.

\paragraph{Downstream probes and selection.}
Each attribute is evaluated with an independent logistic regression fitted
on representations of the 10,000 training images. An internal randomized
80/20 training partition selects $C\in\{0,0.1,1,10\}$ by binary log loss.
$C=0$ denotes an intercept-only predictor. The selected model is refitted on
all training representations. Nonconstant logistic models use L-BFGS,
tolerance $10^{-7}$, and at most 300 iterations. For stochastic channels,
regularization selection uses one draw per training image, followed by
refitting on both draws, avoiding a split of replicas of the same image
between fitting and selection. These CelebA probes reuse representation-training
images. Unlike the dSprites experiment, they do not have an independent
probe-training image partition.

The rate-dependent pretext rule minimizes mean downstream loss on the
separate 1,000-image selection subset, averaged across representation seeds.
The fixed-pretext rule minimizes that selection loss averaged across budgets
and seeds. The Gaussian-predicted rule instead minimizes the fitted Gaussian
surrogate distortion. None of these rules uses test losses. The figures'
explicit specialist/generalist comparison is a specified pair, not a
per-budget test-selected comparison. Reported measured excess NLL subtracts
the fitted uncompressed probe's loss, rather than a known Bayes loss.

The decoder-capacity diagnostic replaces logistic probes with one-hidden-layer
64-unit ReLU classifiers. It uses Adam, learning rate $10^{-3}$, L2 parameter
$10^{-4}$, batches of at most 256, at most 100 epochs, and patience 10 on
checkpoint-validation NLL. Validation improvements smaller than $10^{-8}$
are not counted. Hyperparameters are shared across pretexts.

\paragraph{Uncertainty estimates.}
We use 2,000 paired bootstrap resamples and report pointwise 2.5th and 97.5th percentiles. Gaussian-channel intervals resample test images conditional on the fitted representations and probes, after averaging repeated channel draws within each image. Learned-code intervals independently resample test images and training seeds as crossed factors. Each resample uses the same image indices, and where applicable seed indices, across candidates, budgets, and tasks. Difference intervals are computed from within-resample differences. The decoder-capacity diagnostic uses image resampling only for both compressors. These intervals are not simultaneous confidence bands and do not include uncertainty from drawing new training datasets or repeating model selection.

A separate bootstrap for the Gaussian proxy tasks requests 200 training-image
resamples and refits PCA/whitening and task/pretext moments in each resample,
grouping bootstrap copies of an original image in the same cross-fitting fold.
This is conditional on the frozen backbone and cached view bank. This separate procedure concerns fitted Gaussian surrogate quantities and does not generate the measured-NLL bands in Figure \ref{fig:celeba}.

\paragraph{Crossover and ranking estimators.}
For a generalist-minus-specialist difference curve on the observed rate grid,
the crossover is the first positive-to-negative zero crossing of its
piecewise-linear interpolant. For bootstrap curves, the median and
2.5/97.5 percentiles are computed only over resamples with a crossing.
The crossing fraction and the number of resamples with multiple crossings
are reported separately. No extrapolation beyond the measured grid is used.
Pretexts are ranked by distortion at each nonzero rate. Endpoint instability
is summarized by tie-corrected Kendall $\tau_b$ and the number of discordant
pairs between the lowest and highest nonzero rates. Zero rate is excluded
because all pretexts tie there.

\paragraph{Diagnostics.}
Figures~\ref{fig:app-achieved-rate} and~\ref{fig:app-whitening} assess code
occupancy and the fitted feature coordinates. The plots use retained run
data. No neural networks were retrained to assemble this appendix.

The pair compared in the main text consists of the five-source-attribute
specialist \texttt{joint\_attributes} and the anchor-feature reconstruction
generalist. Figure~\ref{fig:app-full-bank} reports loss curves for every
configured pretext in both compression mechanisms, complementing the
rank views of Figure~\ref{fig:app-rank-bank}.

Figures~\ref{fig:app-per-task} and~\ref{fig:app-per-task-codes}
resolve the aggregate comparison into individual evaluation attributes,
showing that the crossover need not occur uniformly across tasks.
Figure~\ref{fig:app-proxy-fidelity} assesses the association between
Gaussian surrogate distortion and measured excess NLL.
Figure~\ref{fig:app-decoder-capacity} evaluates sensitivity to replacing
linear probes with MLP probes in both compression mechanisms.
Table~\ref{tab:app-theory-data} reports the fitted theoretical quantities
alongside the observed crossovers.

\begin{figure*}[t]
  \centering
  \includegraphics[width=.5\linewidth]{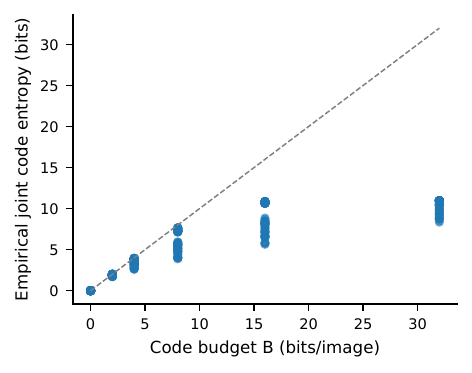}\par\medskip
  \includegraphics[width=.75\linewidth]{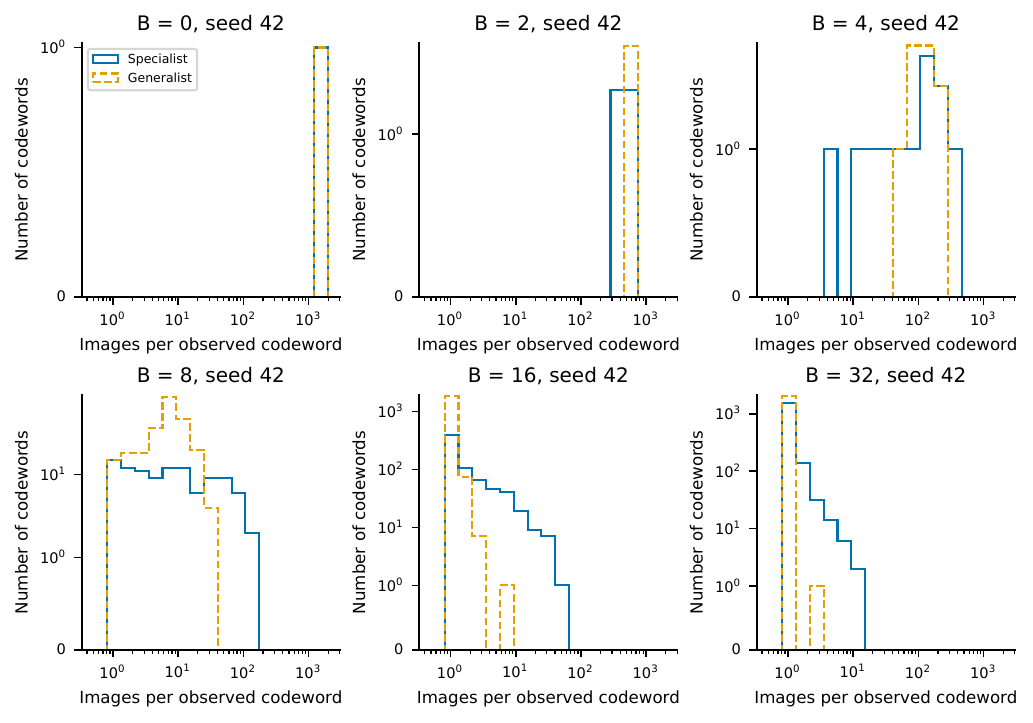}
  \caption{\textbf{Finite-code rate accounting.}
    Top: empirical joint code entropy versus nominal budget for
    300 fitted candidate/budget/seed cells. The identity line
    denotes equality. The maximum observed entropy is 10.96578 bits,
    equal to the finite-test-set ceiling $\log_2(2000)=10.96578$.
    Thus this estimator cannot resolve a 16- or 32-bit population entropy.
    Bottom: test-image codeword-occupancy histograms for the five-attribute
    specialist and reconstruction generalist at all six budgets,
    $B\in\{0,2,4,8,16,32\}$, using seed 42.
    Unobserved codewords are omitted. Sparse occupancy can bias plug-in
    entropy downward. The histograms use logarithmic occupancy and symlog
    count axes to display both singleton and heavily reused codewords.}
  \label{fig:app-achieved-rate}
\end{figure*}

\begin{figure*}[t]
  \centering
  \includegraphics[width=\linewidth]{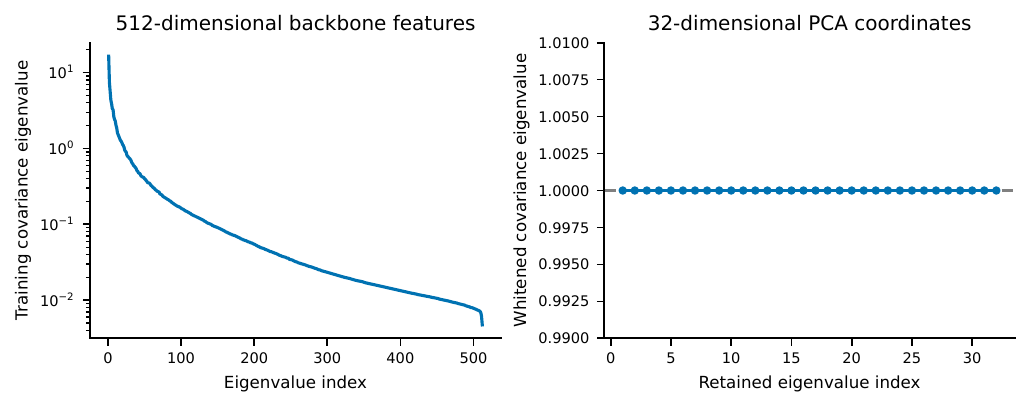}
  \caption{\textbf{Training covariance and whitening.}
    Left: the 512 eigenvalues of the empirical covariance of the 10,000
    training anchor features before PCA, using normalization $1/n$.
    Eigenvalues range from 0.004730 to 16.451719.
    Right: covariance eigenvalues after applying the saved 32-dimensional
    PCA/whitening transform to the same training features. Their maximum
    absolute deviation from one is $5.11\times10^{-15}$, confirming the numerical training-coordinate normalization.}
  \label{fig:app-whitening}
\end{figure*}

\begin{figure*}[t]
  \centering
  \includegraphics[width=0.8\linewidth]{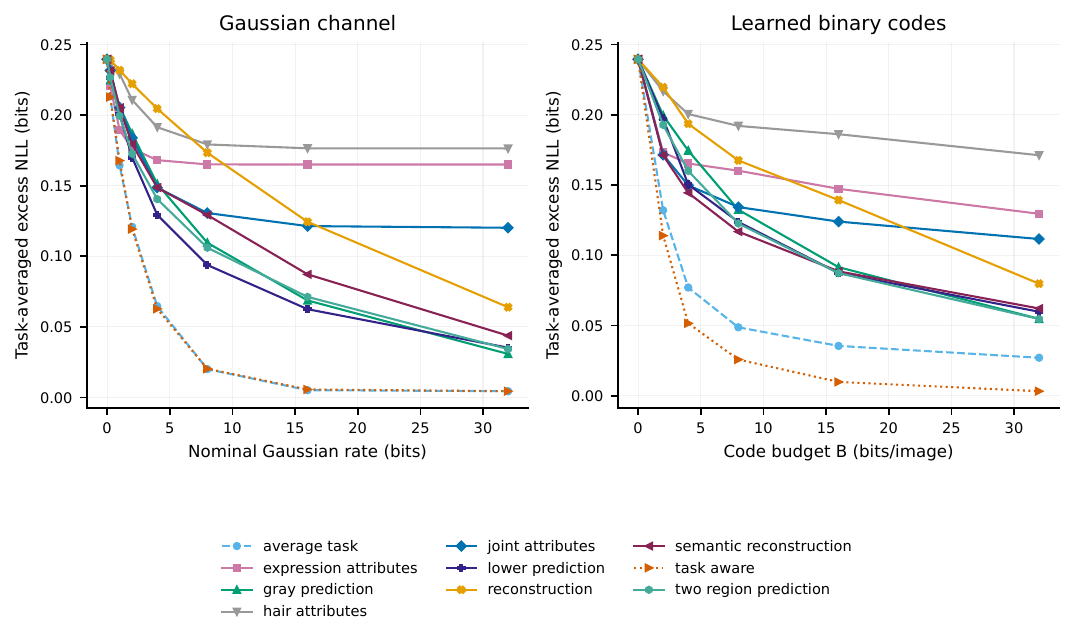}
  \caption{\textbf{Complete CelebA pretext-bank comparison.}
    Both panels show the eight configured pretexts: expression attributes,
    hair attributes, joint attributes, grayscale prediction, lower-region
    prediction, reconstruction, semantic reconstruction, and two-region
    prediction. The task-aware mixture and the single-target Gaussian baseline with information matrix $\overline H$ (“average task”) are also included.
    Curves show five-seed mean task-averaged measured excess NLL relative
    to the uncompressed probe, with markers at each evaluated rate. Bands
    are omitted here for legibility. Colors and markers identify the same
    objectives across panels. The task-aware curve is dotted orange in
    both panels, and the average-task reference is dashed. The left axis
    uses nominal Gaussian rate, the right code budget. The different
    curve shapes reveal objective-dependent changes in relative performance.}
  \label{fig:app-full-bank}
\end{figure*}

\begin{figure*}[t]
  \centering
  \includegraphics[width=.48\linewidth]{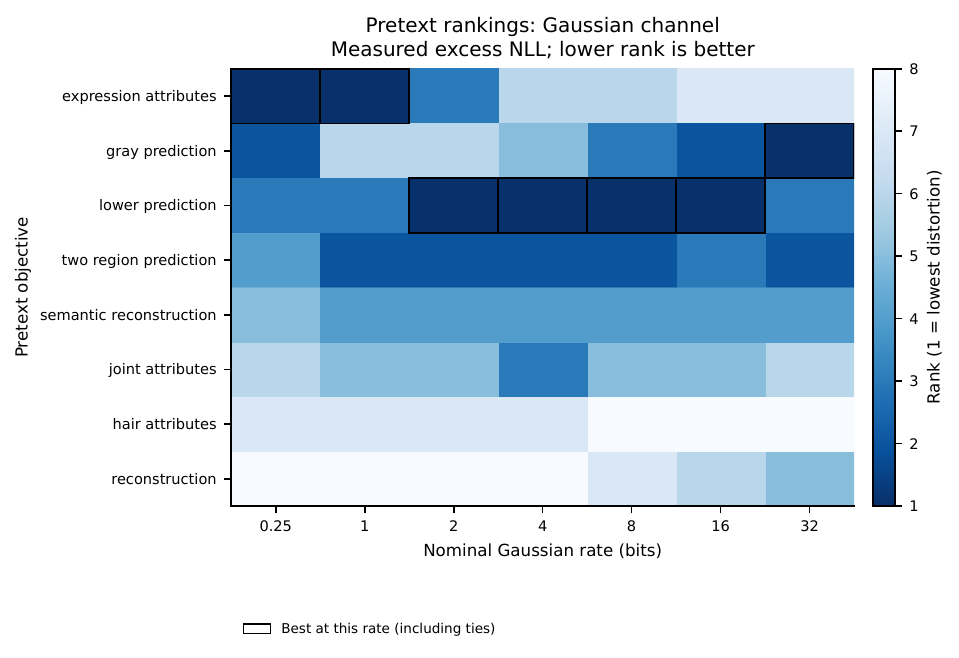}
  \includegraphics[width=.48\linewidth]{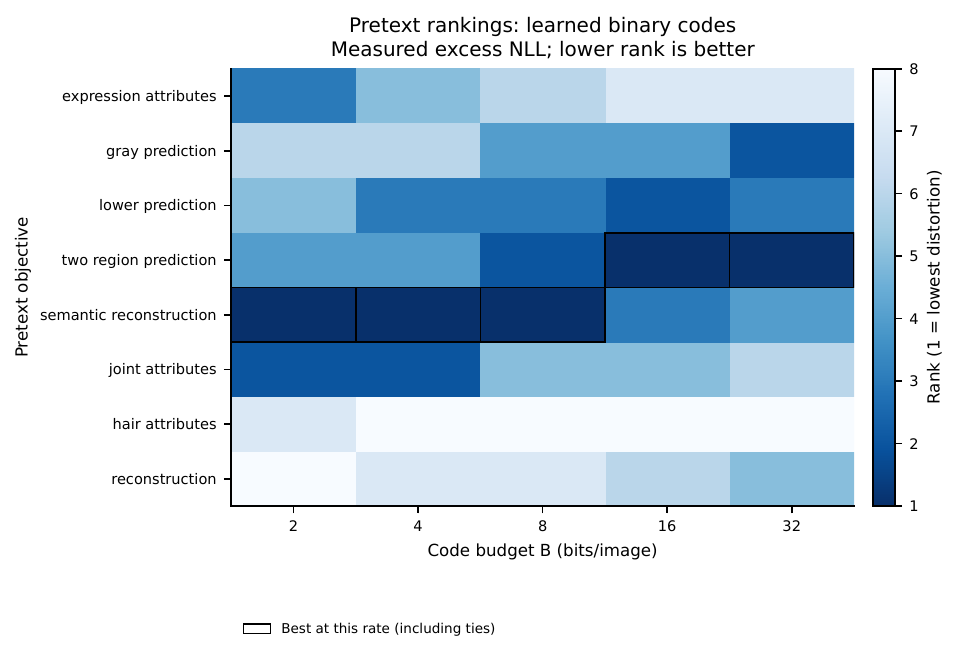}
  \caption{\textbf{Pretext-bank rank instability.} Left: Gaussian representations. Right: learned codes.
    Rows are ordered by the lowest-nonzero-rate rank. Darker cells indicate
    better ranks, and outlines identify column winners including ties.
    Zero rate and the task-aware/average-task references are excluded.
    Gaussian endpoints have $\tau_b=0.357143$, with 9 reversed pairs out of
    28 comparable pairs and a maximum rank change of 6 places. Learned-code endpoints have $\tau_b=0.071429$, with 13 reversed pairs out of 28, and a maximum rank change of 4 places. Thus a pretext's relative performance changes with nominal rate.}
  \label{fig:app-rank-bank}
\end{figure*}

\begin{figure*}[t]
  \centering
  \includegraphics[width=.9\linewidth]{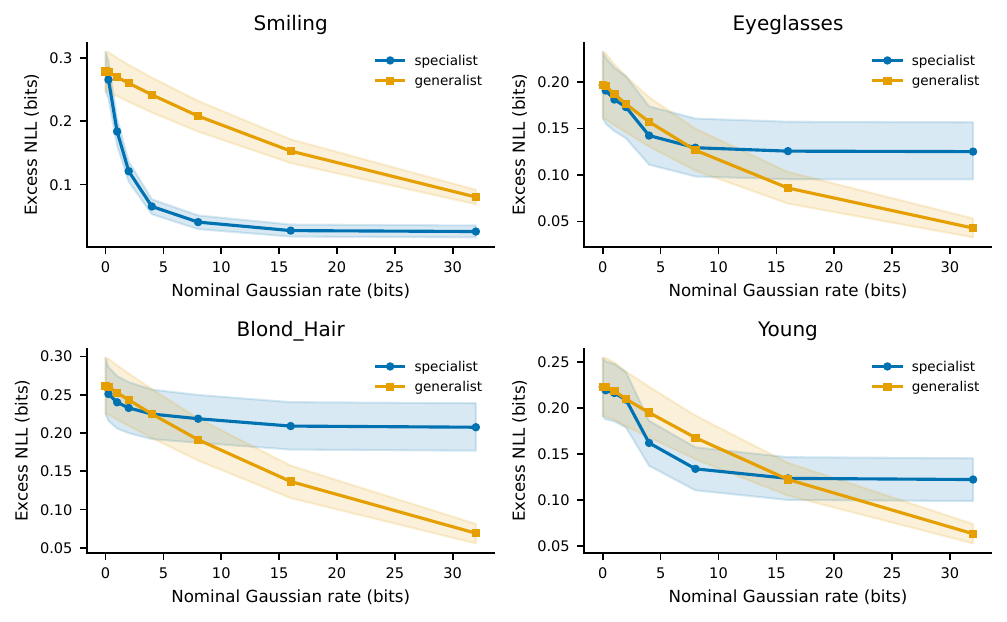}
  \caption{\textbf{CelebA per-attribute Gaussian behavior.}
    Nominal rates are $\{0,0.25,1,2,4,8,16,32\}$ bits. Panels show Smiling,
    Eyeglasses, Blond Hair, and Young. The ordinate is measured excess NLL
    relative to the uncompressed linear probe. Bands are 95\% pointwise
    percentile intervals from 2,000 paired test-image resamples conditional
    on five fitted seeds. At rate 0.25 the specialist has lower mean NLL on
    all four tasks. At rate 32 the generalist is better on Eyeglasses,
    Blond Hair, and Young, while the specialist remains better on Smiling:
    the aggregate reversal is not uniform across tasks.}
  \label{fig:app-per-task}
\end{figure*}

\begin{figure*}[t]
  \centering
  \includegraphics[width=.9\linewidth]{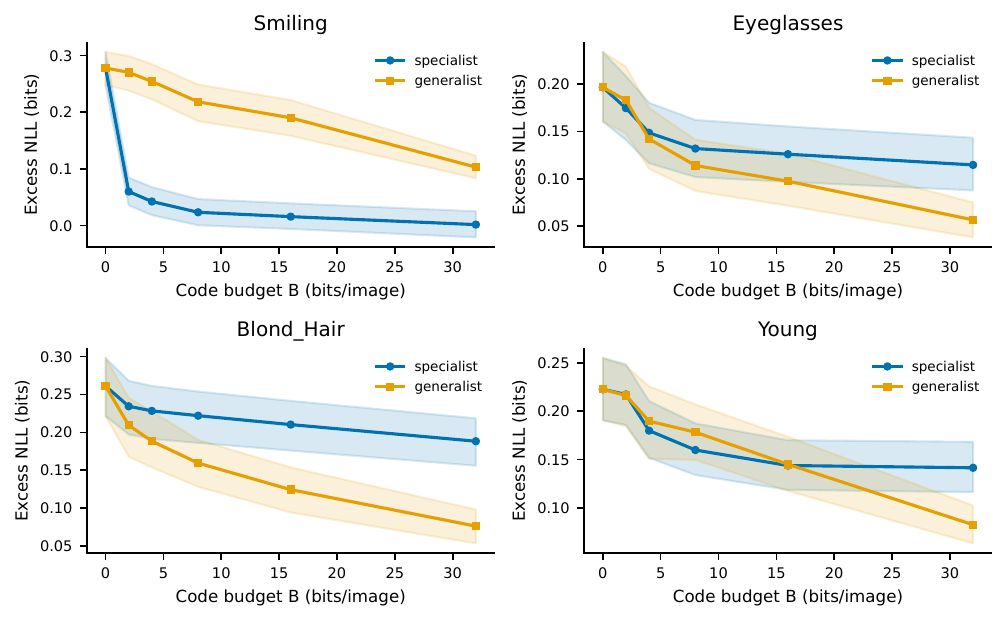}
  \caption{\textbf{CelebA per-attribute learned-code behavior.}
    Code budgets are $B\in\{0,2,4,8,16,32\}$ bits/image. The same four
    evaluation tasks, specialist, generalist, and uncompressed-probe baseline
    are used as in Figure~\ref{fig:app-per-task}. Bands are 95\% pointwise
    percentile intervals from 2,000 paired training-seed/test-image resamples
    using five fitted seeds. These panels expose task-level variation behind
    the aggregate code-budget crossover reported in Table~\ref{tab:app-theory-data}.}
  \label{fig:app-per-task-codes}
\end{figure*}

\begin{figure}[t]
  \centering
  \includegraphics[width=0.5\linewidth]{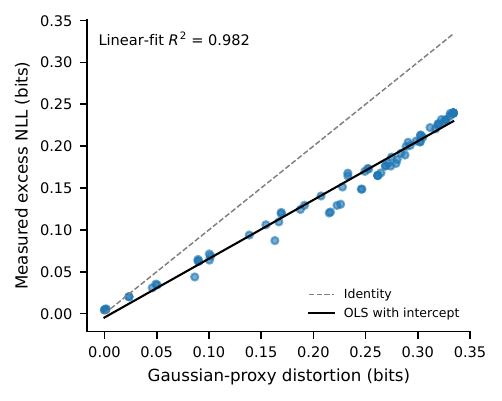}
  \caption{\textbf{Gaussian surrogate fidelity.}
    Ordinary least squares of measured excess NLL on predicted Gaussian
    distortion, with a fitted intercept, gives $R^2=0.9815$ across 80
    objective/rate pairs. The fitted line is
    $\widehat D_{\mathrm{measured}}=0.701817 D_{\mathrm{proxy}}-0.004594$.
    The black line shows this fit. The dashed gray identity line remains
    a calibration reference. The score is descriptive and in-sample, not a
    held-out regression score. Strong linear association coexists with a
    non-unit slope, so it does not imply identity calibration.}
  \label{fig:app-proxy-fidelity}
\end{figure}

\begin{figure*}[t]
  \centering
  \includegraphics[width=0.9\linewidth]{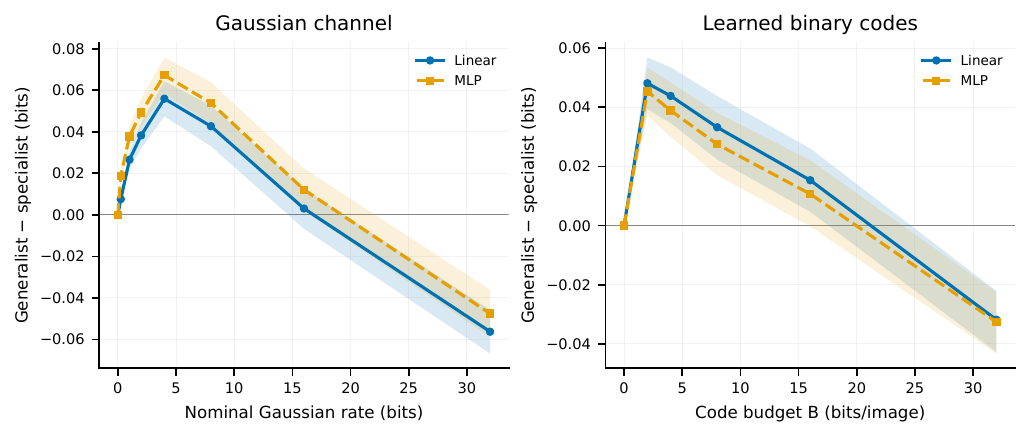}
  \caption{\textbf{Decoder-capacity sensitivity.}
    Left: Gaussian channel. Right: learned binary codes. Linear and
    64-hidden-unit MLP probes evaluate the same representations in each
    mechanism. The ordinate is generalist minus specialist excess NLL.
    The zero line denotes equal performance. Bands in both panels use
    2,000 paired test-image resamples conditional on the five fitted seeds,
    rather than additional seed resampling. Gaussian crossings are 16.8182
    and 19.2453 nominal bits for linear and MLP probes, respectively.
    The code-panel crossings are 21.2035 and 19.9533 bits/image.
    The two rate axes represent different budgets and are not shared.}
  \label{fig:app-decoder-capacity}
\end{figure*}

\begin{table*}[t]
  \centering
  \begin{tabular}{lrl}
\toprule
Quantity & Value & Units / interpretation \\
\midrule
$\lambda_{\max}(\widehat H_{\mathrm{Smiling}})$ & 0.414760 & fitted task SNR \\
$\lambda_{\max}(\widehat H_{\mathrm{Eyeglasses}})$ & 0.737435 & fitted task SNR \\
$\lambda_{\max}(\widehat H_{\mathrm{Blond\_Hair}})$ & 0.884189 & fitted task SNR \\
$\lambda_{\max}(\widehat H_{\mathrm{Young}})$ & 0.375638 & fitted task SNR \\
$d$ & 32 & PCA dimension \\
$\dim S$ & 5 & specialist rank \\
$\Delta$ & 0.0181922 & uniform-subspace reference \\
$D_{\mathrm{miss}}$ & 0.215195 & bits; proxy asymptote \\
Sufficient threshold & 16.2447 & bits; isotropic reference \\
Gaussian crossover & 16.8182 & nominal bits; 95\% CI [14.6479, 19.4241] \\
Code crossover & 21.2035 & bits/image; 95\% CI [16.3215, 25.0128] \\
\bottomrule
\end{tabular}

  \caption{\textbf{Fitted theory-to-data quantities.}
    Values concern the five-source-attribute specialist and reconstruction
    generalist. The first four rows are per-task largest eigenvalues.
    $\Delta=0.0181922$ and $D_{\mathrm{miss}}=0.215195$ are fitted proxy
    quantities. The 16.2447-bit threshold is a sufficient isotropic-reference
    bound, not an equality prediction for the empirical crossover.
    The fitted specialist is anisotropic on its span. The saved strict
    isotropy diagnostic also rejects exact full-space isotropy for the
    reconstruction components. Consequently, the reference threshold is
    not asserted as a certificate for these fitted components.
    Crossover intervals condition on a first positive-to-negative crossing.
    All 2,000 resamples cross in each mechanism, with no multiple crossings.}
  \label{tab:app-theory-data}
\end{table*}

\subsection{dSprites experiment}
\label{app:dsprites-training-details}
\paragraph{Data partitions.}
The primary dSprites experiment uses binary $64\times64$ images from the
official archive, with no resizing or image augmentation. Exact duplicate
pixel arrays are grouped before splitting. Groups are shuffled using
seed $100000+10s$ for experimental seed $s\in\{0,1,2\}$ and assigned intact
to five disjoint folds. Requested sizes are 32,768 representation-training
images, 4,096 representation-validation images, 16,384 probe-training images,
4,096 probe-validation images, and 8,192 test images. Splits are identical across pretexts and $\beta$ values within a seed, but different seeds may reuse images.

The targets are horizontal and vertical position indices, each standardized
using the known uniform 32-level mean $15.5$ and standard deviation
$\sqrt{(32^2-1)/12}$. Normalization is not estimated from held-out labels.
Other rendering factors remain nuisance variation. For $M=4$, task weights
are $(3/4,1/4)$ and independent Gaussian task-noise variances are
$(M^{-2},M^{-4})=(1/16,1/256)$. Task noise is integrated analytically in the
loss, rather than repeatedly sampled as noisy target labels.

\paragraph{Architecture and optimization.}
The image encoder uses three convolutional layers with channels
$1\!\rightarrow\!16\!\rightarrow\!32\!\rightarrow\!32$, kernel sizes
$(5,3,3)$, stride 2, and padding $(2,1,1)$, each followed by SiLU.
The resulting $32\times8\times8$ tensor is flattened and mapped to 128
SiLU units. A linear layer produces the mean and log variance of an
8-dimensional diagonal Gaussian representation. Log variance is clamped
to $[-24,12]$. Representations are sampled by reparameterization.
The task decoder is an $8\!\rightarrow\!128\!\rightarrow\!128\!\rightarrow\!4$
MLP with SiLU hidden activations. It predicts two means and two positive
residual variances, the latter parameterized by softplus plus $10^{-7}$.
The reconstruction decoder is an
$8\!\rightarrow\!128\!\rightarrow\!4096$ MLP with a SiLU hidden layer.

All four pretexts (two specialists, weighted task mixture, and image
reconstruction) minimize
\[
 \mathcal L= D_{\mathrm{pretext}}+
 \beta\,\mathbb E_X\operatorname{KL}_2(q(Z\mid X)\Vert\mathcal N(0,I)).
\]
The task loss for the clean standardized factor value $y_t$, predicted mean $\mu_t$, and total predicted
variance $v_t=\sigma_t^2+v_{\mathrm{res},t}$ is
\[
 \ell_t=\frac{\log(v_t/\sigma_t^2)
       +((y_t-\mu_t)^2+\sigma_t^2)/v_t-1}{2\log 2}.
\]
Specialists minimize their respective task loss. The mixture minimizes
$\tfrac34\ell_1+\tfrac14\ell_2$. Reconstruction minimizes mean Bernoulli
cross-entropy per pixel, converted to bits. Thus reconstruction loss is
normalized by the number of pixels. It is not a summed image log loss.

Each representation is trained from scratch for 100 epochs with Adam,
learning rate $10^{-3}$, batch size 256, no weight decay, and gradient-norm
clipping at 10. The best checkpoint minimizes its own validation
$D_{\mathrm{pretext}}+\beta\mathrm{KL}_2$ objective. Common Monte Carlo
validation draws are reused across epochs by resetting the validation RNG
while restoring the training RNG state afterwards. The regularization grid is
\[
 \beta\in\{0.8,0.4,0.3,0.2,0.15,0.1,0.05,0.02,0.01,
                 0.005,0.002,0.001,0.0005,0.0002\}.
\]
Four pretexts, fourteen regularization values, and three seeds give 168
independent training jobs. The broad $\beta$ grid samples achieved rates on both sides of the potential reversal. Curves are plotted against measured KL, not $\beta$.

\paragraph{Probes and measurement.}
After representation checkpoint selection, the encoder is frozen and a fresh
MLP with the same architecture as the task decoder is trained on the separate
probe-training fold. Probe training uses equal task weights, Adam at
$10^{-3}$, batch size 256, gradient clipping at 10, and 80 epochs. Equal
weighting prevents the probe objective itself from downweighting the rarer
task. The checkpoint with minimum probe-validation NLL is retained.
Probe validation uses eight fixed channel draws per image. Probe training
uses fresh channel samples. The probe and test RNG seeds are $s+30000$
and $s+40000$, respectively.

Test metrics average eight stochastic-channel losses per image, not predictions before evaluating a loss. We report the analytic test-average KL in bits, prevalence-weighted excess NLL, per-factor MSE and $R^2$, and
$\mathcal I_i=-\tfrac12\log_2\mathrm{MSE}_i$.
The KL is a variational upper bound on mutual information.
The recoverability index is a diagnostic transformation of MSE, not an
allocated number of bits for discrete dSprites factors. Its equality to
$-\tfrac12\log_2(1-R_i^2)$ requires unit empirical target variance.
Per-run standard errors cluster duplicate-image groups. Main paper figures
show every seed and across-seed min--max ranges at fixed $\beta$, rather
than interpreting three-seed ranges as confidence intervals.

Table~\ref{tab:dsprites} summarizes mixture-encoder recovery at selected
settings around the reversal.
Figure~\ref{fig:app-dsprites-accounting} compares summed recoverability
indices with the variational rate bound. These indices do not constitute
an additive information decomposition.
Figure~\ref{fig:app-dsprites-context} compares recovery under both
specialists, reconstruction, and the mixture over the full regularization grid.

\paragraph{Diagnostics and controls.}
The primary run contains 168 completed jobs, and the role-swap control run contains 42 completed mixture jobs, with 14 regularization values and seeds $\{0,1,2\}$ in each comparison. The role swap retains 100 representation-training epochs, 80 probe-training epochs, learning rate $10^{-3}$, batch size 256, and an 8-dimensional code.

\paragraph{Role-swap control and crossover audit.}
The role swap exchanges the ordered factors to vertical then horizontal,
retaining task-index weights $(3/4,1/4)$ and noise variances $(1/16,1/256)$.
Figure~\ref{fig:app-role-swap} shows that the recovery ordering reverses
when horizontal and vertical positions exchange task roles.
Crossovers are obtained from the piecewise-linear interpolant on the
observed rate grid, without extrapolation. The primary gap uses the first
positive-to-negative crossing. The swapped plot keeps the physical-factor
gap horizontal minus vertical and uses its first negative-to-positive crossing.
The primary and swapped mean-curve crossings are 5.2638 and
5.1938 bits/image, respectively. All three seeds reverse in each
run. Each seed has exactly one expected-direction crossing.
Table~\ref{tab:app-dsprites-crossovers} gives per-seed estimates and the
adjacent observed-rate brackets used for interpolation. The brackets are
not confidence intervals. The Gaussian surrogate reference is
$R^\star=5.2996$ bits.

\begin{table*}
\centering
    \begin{tabular}{lllll}
        \toprule
        $\mathcal R_{\mathrm{VIB}}$ & $\mathcal I_{\mathrm{common}}$ & $\mathcal I_{\mathrm{precise}}$ & Gap & $\sum_i\mathcal I_i$ \\
        \midrule
        1.75 & 1.85 & 0.00 & 1.84 & 1.85 \\
        2.16 & 2.27 & 0.00 & 2.27 & 2.27 \\
        5.10 & 2.71 & 2.62 & \textbf{0.09} & 5.34 \\
        6.19 & 2.96 & 3.47 & \textbf{-0.52} & 6.43 \\
        7.15 & 3.29 & 4.03 & -0.75 & 7.32 \\
        8.38 & 3.74 & 4.62 & -0.87 & 8.36 \\
        \bottomrule
    \end{tabular}
    \caption{\textbf{dSprites factor recovery.} The stated results are for $M = 4$, averaged over seeds 0,1, and 2.}
    \label{tab:dsprites}
\end{table*}

\begin{table*}[t]
  \centering
  \begin{tabular}{llrrr}
\toprule
Run & Seed & Crossover & Lower grid rate & Upper grid rate \\
\midrule
Primary & 0 & 5.2337 & 5.0378 & 6.1546 \\
Primary & 1 & 5.2432 & 5.1039 & 6.2125 \\
Primary & 2 & 5.3144 & 5.1595 & 6.2178 \\
Role swap & 0 & 5.1481 & 2.1223 & 5.1627 \\
Role swap & 1 & 5.1855 & 5.1095 & 6.1823 \\
Role swap & 2 & 5.2495 & 5.1796 & 6.1687 \\
\bottomrule
\end{tabular}

  \caption{\textbf{Per-seed dSprites crossover audit.}
    Rates are measured variational KL bounds in bits/image. The primary
    gap crosses positive-to-negative. The role-swap physical-factor gap
    crosses negative-to-positive. Every seed has one such crossing.
    Brackets are observed neighboring grid rates, not uncertainty intervals.}
  \label{tab:app-dsprites-crossovers}
\end{table*}

\begin{figure}[t]
  \centering
  \includegraphics[width=0.5\linewidth]{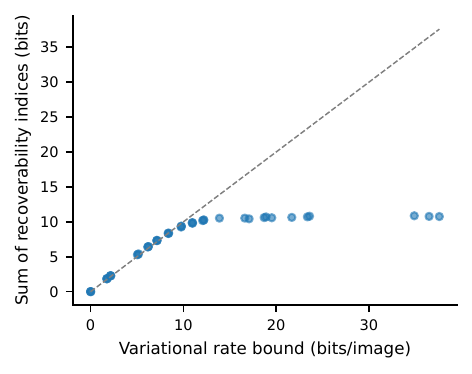}
  \caption{\textbf{dSprites recoverability accounting.}
    The 42 mixture runs compare the sum of per-run MSE indices with the
    variational KL bound over the full 14-setting grid. Maximum positive
    index-sum excess is 0.2470 bits. The sum is a recovery diagnostic, not
    a quantitative decomposition of the rate bound. The identity line is
    retained as a visual reference. No linear relationship or
    information-theoretic equality is assumed for these discrete factors.}
  \label{fig:app-dsprites-accounting}
\end{figure}

\begin{figure}[t]
  \centering
  \includegraphics[width=0.5\linewidth]{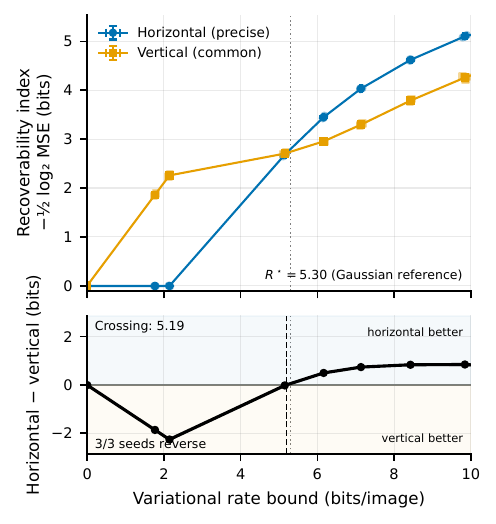}
  \caption{\textbf{dSprites role-swap control.}
    Vertical position is common/moderate and horizontal position rare/precise.
    Individual points and thin gap curves show three seeds. Central curves
    show means and error bars give seed min--max ranges at fixed $\beta$.
    The horizontal-minus-vertical gap crosses negative-to-positive at
    5.1938 bits/image. All three seeds reverse. The displayed
    view covers 0--10 bits/image. The dotted reference is
    $R^\star=5.2996$ from the Gaussian surrogate, not an exact
    dSprites prediction. The changed physical-factor ordering supports a
    task-role effect rather than an explanation based solely on one fixed
    spatial direction being harder to recover.}
  \label{fig:app-role-swap}
\end{figure}

\begin{figure*}[t]
  \centering
  \includegraphics[width=0.9\linewidth]{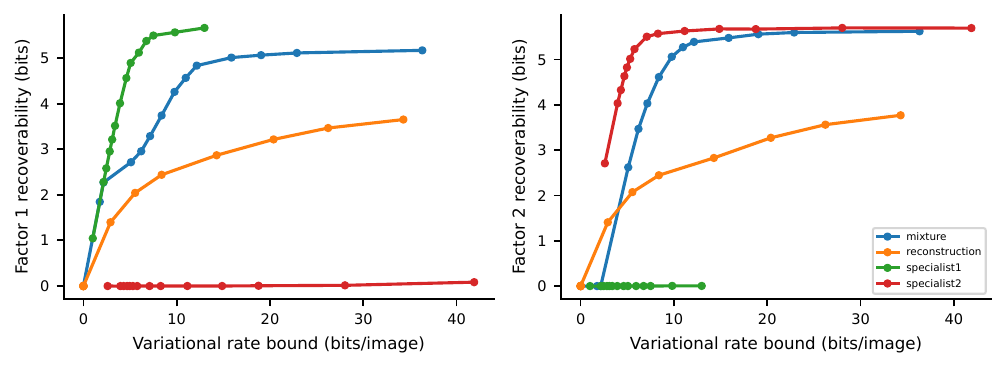}
  \caption{\textbf{Recovery under the full dSprites bank.}
    The two panels show horizontal and vertical recoverability versus
    measured KL for both specialists, reconstruction, and mixture.
    Lines connect three-seed means at each of 14 settings. No confidence
    bands are drawn. At each objective's highest mean achieved rate, the
    horizontal specialist has index pair $(5.663,0.007)$, the vertical
    specialist $(0.081,5.697)$, reconstruction $(3.651,3.773)$, and mixture
    $(5.171,5.626)$. These endpoints occur at different achieved rates and
    are not a matched-rate performance comparison. They illustrate that
    specialists primarily recover their own factor, while reconstruction
    and mixture retain both.}
  \label{fig:app-dsprites-context}
\end{figure*}

\subsection{Gaussian covariance optimization}
\label{app:gaussian-solver}
In this subsection, we explain the algorithm used to compute the optimal Gaussian channel in Proposition \ref{prop:gaussian}.

\paragraph{Program and relevant subspace.}
For positive semidefinite task matrices $H_t\in\mathbb R^{d\times d}$ and
nonnegative weights $w_t$ summing to one, the Gaussian surrogate program is
\begin{equation}
 \min_{0\prec K\preceq I}\quad
 D(K)=\frac{1}{2\log 2}\sum_{t=1}^T w_t
          \log\det(I+F_t^\top K F_t),
 \qquad -\frac{\log\det K}{2\log 2}\le R,
 \label{eq:implemented-covariance-program}
\end{equation}
where $H_t=F_tF_t^\top$. The determinant expression is equivalent to
$\log\det(I+H_t^{1/2}KH_t^{1/2})$ but permits low-rank factors.
Here $K$ is the residual covariance of the whitened Gaussian source given
its representation. Scalar regression targets give rank-one $H_t$.

Before the iterative solve, the code computes a basis $U$ for the numerically
relevant span of $\sum_t w_tH_t$, keeping eigenvalues larger than
$\max(10^{-12},10^{-10}\lambda_{\max})$. It solves in this span, of dimension
$q\le d$, using $\widetilde H_t=U^\top H_tU$, and returns
$K=I+U(\widetilde K-I_q)U^\top$. Irrelevant directions receive no rate.
For the scalar-target task-aware and five-attribute specialist objectives,
$q\le4$ and $q\le5$, respectively. With a zero relevant span the code returns
$K=I$ without iteration. The span truncation is numerical. Certificates
refer to the retained, numerically represented program.

\paragraph{Log-covariance coordinates and projected steps.}
Within the relevant span, the solver parameterizes $K=\exp(-S)$.
It optimizes on
\[
 \mathcal S_R=\{S\succeq0:\operatorname{tr}S=2R\log2\}.
\]
The objective is nondecreasing in $K$, so an optimum can exhaust the rate
budget when the relevant span is nonempty. This parameterization enforces
$K\preceq I$ and the rate constraint without a penalty parameter.
It does not assert ordinary Euclidean convexity of $D(\exp(-S))$.
At positive rate, the main experiment initializes
$S_0=(2R\log2/q)I_q$ separately for each solve. Although the low-level solver
accepts a warm start, the main task/pretext wrappers do not supply one.
At zero rate it returns $K=I_q$ directly.

Let $G_K=\nabla_KD(K)$. The implementation evaluates
\[
 G_K=\frac{1}{2\log2}\sum_t w_t
 F_t(I+F_t^\top K F_t)^{-1}F_t^\top.
\]
If $S=V\operatorname{diag}(s)V^\top$, its gradient is obtained using the
Fr\'echet derivative of the matrix exponential:
\[
 \nabla_SD(\exp(-S))=-V\bigl[C\odot(V^\top G_KV)\bigr]V^\top,
 \quad
 C_{ij}=\begin{cases}
 e^{-s_i},&s_i=s_j,\\
 e^{-\min(s_i,s_j)}\dfrac{1-e^{-|s_i-s_j|}}{|s_i-s_j|},&s_i\ne s_j.
 \end{cases}
\]
The code evaluates the divided difference with \texttt{expm1} and uses its
limiting value when $|s_i-s_j|\le10^{-10}$.

Each candidate step is
$S^+=\Pi_{\mathcal S_R}(S-\eta\nabla_SD)$.
Projection diagonalizes the symmetric argument and projects its eigenvalues
onto the nonnegative simplex of mass $2R\log2$ using sorting and a threshold.
The initial step size is 10. Subsequent iterations use the
Barzilai--Borwein proposal
$\eta=\|\Delta S\|_F^2/\langle\Delta S,\Delta G_S\rangle$
when the denominator exceeds $10^{-18}$, clipped to $[10^{-3},10^7]$.
Backtracking halves the step until
\[
 D(e^{-S^+})\le D(e^{-S})+
 \langle\nabla_SD,S^+-S\rangle+
 \frac{\|S^+-S\|_F^2}{2\eta}+10^{-15}.
\]
At most 60 trials are allowed per iteration. If this limit is exhausted,
the implementation still advances to its last candidate. Convergence must
therefore be assessed from the final certificate, not assumed from the line
search alone.

\paragraph{Stopping certificate.}
The stopping rule uses a geodesic Lagrangian residual, rather than objective
change or a fixed iteration count. Write $K=V\operatorname{diag}(k)V^\top$,
$Q=K^{1/2}$, and $M=2\log2\,QG_KQ$.
Eigen-directions with $k_i<1-10^{-7}$ are classified as active. The rate
multiplier $\lambda$ is the nonnegative mean of the active diagonal entries
of $V^\top MV$. If there are no active directions it is the nonnegative
largest eigenvalue of $M$. On the inactive subspace, the code forms the
positive-semidefinite part of
$(\lambda I-(V^\top MV)_{\mathrm{inactive}})/(2\log2)$
and lifts it to a dual matrix $Z\succeq0$, zero on active directions.
Define
\[
 E=\frac{M-\lambda I}{2\log2}+QZQ,
 \qquad r(K)=-\frac{\log\det K}{2\log2}.
\]
The reported residual-based gap bound is
\begin{equation}
 \widehat\epsilon=
 \max\{0,\lambda(R-r(K))+\operatorname{tr}[Z(I-K)]\}
 +2R\log2\sqrt q\,\|E\|_F.
 \label{eq:implemented-gap}
\end{equation}
This is the implementation's a posteriori certificate for the Gaussian
surrogate program. It is not a statistical error bound for learned image
representations. It is computed at iterations $1,11,21,\ldots$ and again
on the returned iterate. The main tolerance is $10^{-5}$ bits of weighted
surrogate distortion, with at most 3,000 outer iterations. Main-run wrappers
reject an unconverged solve. In the task-aware run, all seven
positive-rate solves report 11 iterations and gap bounds below $10^{-5}$.
This is an empirical observation, not an iteration-complexity guarantee.

\paragraph{Single-matrix objectives.}
A pretext represented by one matrix bypasses projected iterations.
For $H=V\operatorname{diag}(h)V^\top$, the code computes
$K=V\operatorname{diag}(k)V^\top$ with
$k_i=\min(1,e^a/h_i)$ on positive eigenvalues and $k_i=1$ otherwise.
Brent's method solves the scalar equation
\[
 \frac12\sum_{i:h_i>0}\max\left\{\frac{\log h_i-a}{\log2},0\right\}=R
\]
in log-water-level coordinates, with absolute tolerance $10^{-12}$.
This avoids loss of relative accuracy when the water level is small.
A multivariate attribute pretext with several scalar task matrices uses the
weighted multi-task solver, not this single-matrix shortcut.

\paragraph{Time complexity.}
Let $q$ be the retained dimension, $T$ the number of matrices in one objective,
$r_t=\operatorname{rank}(\widetilde H_t)$ after numerical factorization,
$N$ the number of outer iterations, and $L\le60$ the maximum number of
backtracking trials used in an iteration. Dense symmetric eigendecompositions,
matrix-exponential reconstruction, gradient basis changes, and PSD/simplex
projection cost $O(q^3)$ each. Simplex sorting adds $O(q\log q)$.
With factored tasks, an objective/gradient evaluation costs
\[
 \mathcal C_H=
 O\!\left(\sum_{t=1}^T(q^2r_t+qr_t^2+r_t^3)\right).
\]
When all tasks have rank at most one, the vectorized implementation reduces
this to $O(Tq^2)$. Certificate evaluation adds $O(q^3+\mathcal C_H)$.
Although backtracking requests only an objective value, the
\texttt{value} routine also computes its gradient, so the same
$\mathcal C_H$ cost applies to those calls. Consequently, one solve has
iteration cost
\[
 O\!\left(N(1+L)(q^3+\mathcal C_H)\right),
\]
or $O(N(1+L)(q^3+Tq^2))$ for rank-one families and
$O(N(1+L)Tq^3)$ in the full-rank case. These are arithmetic-operation bounds
conditional on $N$ and $L$. The implementation establishes no worst-case
bound on $N$ as a function of the requested gap tolerance.

Preprocessing also matters. Forming the weighted matrix and its relevant-span
basis costs $O(Td^2+d^3)$. In the wrappers, the direct, unoptimized
three-operand \texttt{einsum} for $U^\top H_tU$ has contraction count
$O(Td^2q^2)$. A staged multiplication could reduce this to
$O(T(d^2q+dq^2))$, but that optimization is not implemented here.
Factorizing the reduced matrices costs $O(Tq^3)$. Lifting the solution costs
$O(dq^2+d^2q)$. These setup operations are repeated by the main wrappers
at each rate rather than cached across the grid.
The single-matrix shortcut costs $O(d^3+Jd)$, where $J$ is the number of
scalar root-function evaluations, including dense reconstruction of $K$.

\paragraph{Space complexity and repeated solves.}
The solver retains dense task matrices even for rank-one families. Its
resident matrix storage is therefore $O(Tq^2)$, with factors requiring
$O(q\sum_t r_t)$ and a constant number of $q\times q$ work arrays and
iterates. Including the original matrices and lifted covariance in the
wrappers gives $O(Td^2+Tq^2+dq)$ storage for one solve. It would be incorrect
to describe the present rank-one implementation as using only $O(Tq)$
space. Storing solutions for $G$ rates additionally requires $O(Gd^2)$
per candidate. The solve itself does not depend on the number of training
images once the task matrices have been formed. Fitting features, moments,
and neural representations has separate costs. Rate grids and bootstrap
replications repeat the solve, so their work is additive over all
candidate/rate/resample combinations, including the corresponding refitting
costs for the moment bootstrap.

\end{document}